\documentclass[10pt]{article}
\usepackage{slca_style}
\pdfmapline{+optimistic < assets/Optimistic.ttf <T1-WGL4.enc}
\pdfmapline{+optimistic_ts1 < assets/Optimistic.ttf <q-ts1-uni.enc}
\DeclareFontFamily{T1}{optimistic}{}
\DeclareFontFamily{TS1}{optimistic}{}
\DeclareFontShape{T1}{optimistic}{m}{n}{<-> s * [0.88] assets/optimistic}{}
\DeclareFontShape{T1}{optimistic}{b}{n}{<-> s * [0.88] assets/optimistic}{}
\DeclareFontShape{T1}{optimistic}{m}{it}{<-> s * [0.88] assets/optimistic}{}
\DeclareFontShape{T1}{optimistic}{b}{it}{<-> s * [0.88] assets/optimistic}{}
\DeclareFontShape{T1}{optimistic}{bx}{n}{<-> s * [0.88] assets/optimistic}{}
\DeclareFontShape{T1}{optimistic}{bx}{it}{<-> s * [0.88] assets/optimistic}{}
\DeclareFontShape{TS1}{optimistic}{m}{n}{<-> s * [0.88] assets/optimistic_ts1}{}
\DeclareFontShape{TS1}{optimistic}{b}{n}{<-> s * [0.88] assets/optimistic_ts1}{}
\DeclareFontShape{TS1}{optimistic}{m}{it}{<-> s * [0.88] assets/optimistic_ts1}{}
\DeclareFontShape{TS1}{optimistic}{b}{it}{<-> s * [0.88] assets/optimistic_ts1}{}
\DeclareFontShape{TS1}{optimistic}{bx}{n}{<-> s * [0.88] assets/optimistic_ts1}{}
\DeclareFontShape{TS1}{optimistic}{bx}{it}{<-> s * [0.88] assets/optimistic_ts1}{}

\DisableLigatures{family=optimistic}
\makeatletter
\newif\ifBA@rulepending
\newif\ifBA@rowfollows
\newdimen\BA@rulewd
\newdimen\BA@rulegap
\newdimen\BA@rowht
\newdimen\BA@ruledist
\newdimen\BA@savedepth
\let\BA@rulecolor\relax
\let\BA@CT@setup\CT@setup
\def\CT@setup{%
  \dimen@\ht\z@
  \ifdim\minrowclearance>\z@\advance\dimen@\minrowclearance\fi
  \ifdim\dimen@>\BA@rowht\global\BA@rowht\dimen@\fi
  \BA@CT@setup}
\def\BA@drawrule{%
  \BA@savedepth\prevdepth
  \kern-\BA@ruledist
  \kern-\BA@rulewd
  {\BA@rulecolor\hrule\@height\BA@rulewd}%
  \kern\BA@ruledist
  \prevdepth\BA@savedepth
  \global\BA@rulependingfalse}
\def\BA@flushrule{%
  \ifBA@rulepending
    \ifdim\BA@rowht>\z@
      \PackageWarning{main}{A deferred table rule was flushed without a row
        end;\MessageBreak its position may be wrong}%
    \fi
    \BA@ruledist\BA@rulegap
    \BA@drawrule
  \fi}
\def\BA@rowend{%
  \ifBA@rulepending
    \BA@ruledist\BA@rowht
    \ifdim\ht\@arstrutbox>\BA@ruledist\BA@ruledist\ht\@arstrutbox\fi
    \advance\BA@ruledist\prevdepth
    \advance\BA@ruledist\BA@rulegap
    \BA@drawrule
  \fi
  \global\BA@rowht\z@}
\def\BA@checknext{%
  \BA@rowfollowstrue
  \ifx\@tempa\end\BA@rowfollowsfalse\fi
  \ifx\@tempa\toprule\BA@rowfollowsfalse\fi
  \ifx\@tempa\midrule\BA@rowfollowsfalse\fi
  \ifx\@tempa\bottomrule\BA@rowfollowsfalse\fi
  \ifx\@tempa\cmidrule\BA@rowfollowsfalse\fi
  \ifx\@tempa\specialrule\BA@rowfollowsfalse\fi
  \ifx\@tempa\addlinespace\BA@rowfollowsfalse\fi
  \ifx\@tempa\morecmidrules\BA@rowfollowsfalse\fi
  \ifx\@tempa\hline\BA@rowfollowsfalse\fi
  \ifx\@tempa\cline\BA@rowfollowsfalse\fi
  \ifx\@tempa\noalign\BA@rowfollowsfalse\fi
  \ifx\@tempa\egroup\BA@rowfollowsfalse\fi}
\def\@BTnormal{\futurenonspacelet\@tempa\BA@BTnormal}
\def\BA@BTnormal{%
  \BA@checknext
  \ifBA@rowfollows
    \kern\@thisrulewidth
    \prevdepth-\@m\p@
    \global\BA@rulewd\@thisrulewidth
    \global\BA@rulegap\z@
    \global\let\BA@rulecolor\CT@arc@
    \global\BA@rulependingtrue
  \else
    {\CT@arc@\hrule\@height\@thisrulewidth}%
  \fi
  \@BTendrule}
\let\BA@BTrule\@BTrule
\def\@BTrule[#1]{\BA@flushrule\BA@BTrule[#1]}
\patchcmd\@BTendrule{\vskip\@belowrulesep}
  {\ifBA@rulepending\global\advance\BA@rulegap\@belowrulesep\fi
   \vskip\@belowrulesep}
  {}{\PackageError{main}{Could not patch booktabs' \string\@BTendrule}{}}
\let\BA@cmidrule\cmidrule
\def\cmidrule{\noalign{\BA@flushrule}\BA@cmidrule}
\def\BA@splitendarray#1#2#3\BA@stop{%
  \expandafter\ifx\csname tbl_crcr:n\endcsname#1%
    \def\endarray{#1{#2}\noalign{\BA@flushrule}#3}%
  \else
    \PackageError{main}{Unexpected definition of \string\endarray}{}%
  \fi}
\expandafter\BA@splitendarray\endarray\BA@stop
\def\BA@tablestart{%
  \edef\BA@outerdims{%
    \global\BA@rowht\the\BA@rowht\relax
    \global\BA@rulewd\the\BA@rulewd\relax
    \global\BA@rulegap\the\BA@rulegap\relax}%
  \let\BA@outercolor\BA@rulecolor
  \ifBA@rulepending\let\BA@outerpending\BA@rulependingtrue
  \else\let\BA@outerpending\BA@rulependingfalse\fi
  \global\BA@rulependingfalse
  \global\BA@rowht\z@
  \CT@everycr\expandafter{\the\CT@everycr\noalign{\BA@rowend}}}
\def\BA@tableend{%
  \BA@outerdims
  \global\let\BA@rulecolor\BA@outercolor
  \global\BA@outerpending}
\AtBeginDocument{%
  \expandafter\def\expandafter\@tabarray\expandafter{%
    \expandafter\BA@tablestart\@tabarray}%
  \expandafter\def\expandafter\endarray\expandafter{\endarray\BA@tableend}}
\makeatother
\makeatletter
\newif\ifBA@cmidpending
\newif\ifBA@cmidskip
\newif\ifBA@cmidemit
\newcount\BA@cmidserial
\newdimen\BA@lastrowht
\newdimen\BA@lastrowdepth
\newdimen\BA@cmidemitrowht
\newdimen\BA@cmidemitrowdepth
\def\BA@cmidqueue{}
\let\BA@orig@cmidrulea\@cmidrulea
\let\BA@orig@cmidruleb\@cmidruleb
\let\BA@orig@cmidrule\cmidrule
\def\cmidrule{\noalign{\global\BA@cmidskiptrue\BA@flushrule}\BA@orig@cmidrule}
\def\BA@savecmid#1{%
  \global\advance\BA@cmidserial\@ne
  \edef\BA@cmidid{\the\BA@cmidserial}%
  \expandafter\xdef\csname BA@cmidtype@\BA@cmidid\endcsname{#1}%
  \expandafter\xdef\csname BA@cmidla@\BA@cmidid\endcsname{\the\@cmidla}%
  \expandafter\xdef\csname BA@cmidlb@\BA@cmidid\endcsname{\the\@cmidlb}%
  \expandafter\xdef\csname BA@cmidwidth@\BA@cmidid\endcsname{\the\@thisrulewidth}%
  \expandafter\xdef\csname BA@cmidkernl@\BA@cmidid\endcsname{\the\cmrkern@l}%
  \expandafter\xdef\csname BA@cmidkernr@\BA@cmidid\endcsname{\the\cmrkern@r}%
  \expandafter\global\expandafter\let\csname BA@cmidcolor@\BA@cmidid\endcsname\CT@arc@
  \edef\BA@cmidtmp{\noexpand\BA@cmidone{\BA@cmidid}}%
  \expandafter\g@addto@macro\expandafter\BA@cmidqueue\expandafter{\BA@cmidtmp}%
  \global\BA@cmidpendingtrue
  \kern\@thisrulewidth
}
\def\BA@capturecmid#1{\omit\cr\noalign{\BA@savecmid{#1}}}
\def\@cmidrulea{\BA@capturecmid{a}}
\def\@cmidruleb{\BA@capturecmid{b}}
\def\BA@cmidone#1{%
  \noalign{%
    \global\@cmidla=\csname BA@cmidla@#1\endcsname\relax
    \global\@cmidlb=\csname BA@cmidlb@#1\endcsname\relax
    \global\@thisrulewidth=\csname BA@cmidwidth@#1\endcsname\relax
    \global\cmrkern@l=\csname BA@cmidkernl@#1\endcsname\relax
    \global\cmrkern@r=\csname BA@cmidkernr@#1\endcsname\relax
    \expandafter\global\expandafter\let\expandafter\CT@arc@\csname BA@cmidcolor@#1\endcsname
    \if a\csname BA@cmidtype@#1\endcsname
      \global\let\BA@cmidnext\BA@orig@cmidrulea
    \else
      \global\let\BA@cmidnext\BA@orig@cmidruleb
    \fi
    \BA@ruledist\BA@cmidemitrowht
    \advance\BA@ruledist\BA@cmidemitrowdepth
    \advance\BA@ruledist\@thisrulewidth
    \kern-\BA@ruledist}%
  \BA@cmidnext
  \noalign{\kern\BA@cmidemitrowht\kern\BA@cmidemitrowdepth}%
}
\def\BA@cmidemitpending#1#2#3{%
  #1#2#3%
  \noalign{%
    \global\BA@cmidemittrue
    \global\BA@cmidemitrowht\BA@lastrowht
    \global\BA@cmidemitrowdepth\BA@lastrowdepth}%
  \cr
  \noalign{\kern-\BA@cmidemitrowht\kern-\BA@cmidemitrowdepth}\BA@cmidqueue\noalign{\global\BA@cmidpendingfalse\global\let\BA@cmidqueue\@empty\global\BA@cmidemitfalse}%
}
\def\BA@cmidemitter{%
  \ifBA@cmidemit
  \else\ifBA@cmidskip
    \noalign{\global\BA@cmidskipfalse}%
  \else\ifBA@cmidpending
    \BA@cmidemitpending
  \fi\fi\fi}
\let\BA@oldrowend\BA@rowend
\def\BA@rowend{%
  \global\BA@lastrowht\BA@rowht
  \ifdim\ht\@arstrutbox>\BA@lastrowht\global\BA@lastrowht\ht\@arstrutbox\fi
  \global\BA@lastrowdepth\prevdepth
  \BA@oldrowend}
\def\BA@tablestart{%
  \edef\BA@outerdims{%
    \global\BA@rowht\the\BA@rowht\relax
    \global\BA@rulewd\the\BA@rulewd\relax
    \global\BA@rulegap\the\BA@rulegap\relax}%
  \let\BA@outercolor\BA@rulecolor
  \ifBA@rulepending\let\BA@outerpending\BA@rulependingtrue
  \else\let\BA@outerpending\BA@rulependingfalse\fi
  \global\BA@rulependingfalse
  \global\BA@rowht\z@
  \CT@everycr\expandafter{\the\CT@everycr\noalign{\BA@rowend}\BA@cmidemitter}}
\makeatother

\theoremstyle{plain}
\newtheorem{theorem}{Theorem}[section]
\newtheorem{proposition}[theorem]{Proposition}

\newtheorem{corollary}[theorem]{Corollary}
\theoremstyle{definition}

\theoremstyle{remark}

\newcommand{\bx}{\mathbf{x}}
\newcommand{\mr}[1]{m(\mathbf{x}_{#1})}
\newcommand{\reward}[1]{r(\mathbf{x}_{#1})}
\newcommand{\seqlen}{L}
\newcommand{\Tsteps}{T}
\newcommand{\Nsamples}{N}
\newcommand{\Kcands}{K}
\newcommand{\BranchEvery}{b}
\newcommand{\NSnapshots}{1{,}012{,}992}
\newcommand{\NTestTrajectories}{42{,}208}
\crefname{assumption}{Assumption}{Assumptions}
\Crefname{assumption}{Assumption}{Assumptions}

\begin{document}
\sloppy
\makearxivtitle
\begingroup
\renewcommand{\thefootnote}{\ensuremath{\dagger}}
\footnotetext{Correspondence to: Zhijun Gao (\texttt{gaozhijun@pku.edu.cn}).}
\endgroup
\section{Introduction}
\label{sec:intro}

Discrete diffusion language models (dLLMs) are emerging as an alternative reasoning
paradigm to autoregressive (AR) decoding. Rather than committing to one
token prefix at a time, a dLLM maintains a partially masked full sequence and gradually
denoises it into a solution. For mathematical reasoning, these intermediate states, which
we call snapshots, can already contain tentative equations, quantities, answer fragments,
and later reasoning steps before the whole derivation is finalized. This makes dLLM
generation look unusually well matched to process reward models (PRMs): a reward model could inspect the evolving
solution state, identify promising trajectories, and steer computation before the final
answer is fixed. If this works, PRM guidance would offer a natural way to convert extra
test-time compute into better reasoning accuracy.

\begin{figure}[!t]
    \centering
    \begin{tikzpicture}
    \node[anchor=south west,inner sep=0] (fw) at (0,0) {\includegraphics[width=\textwidth]{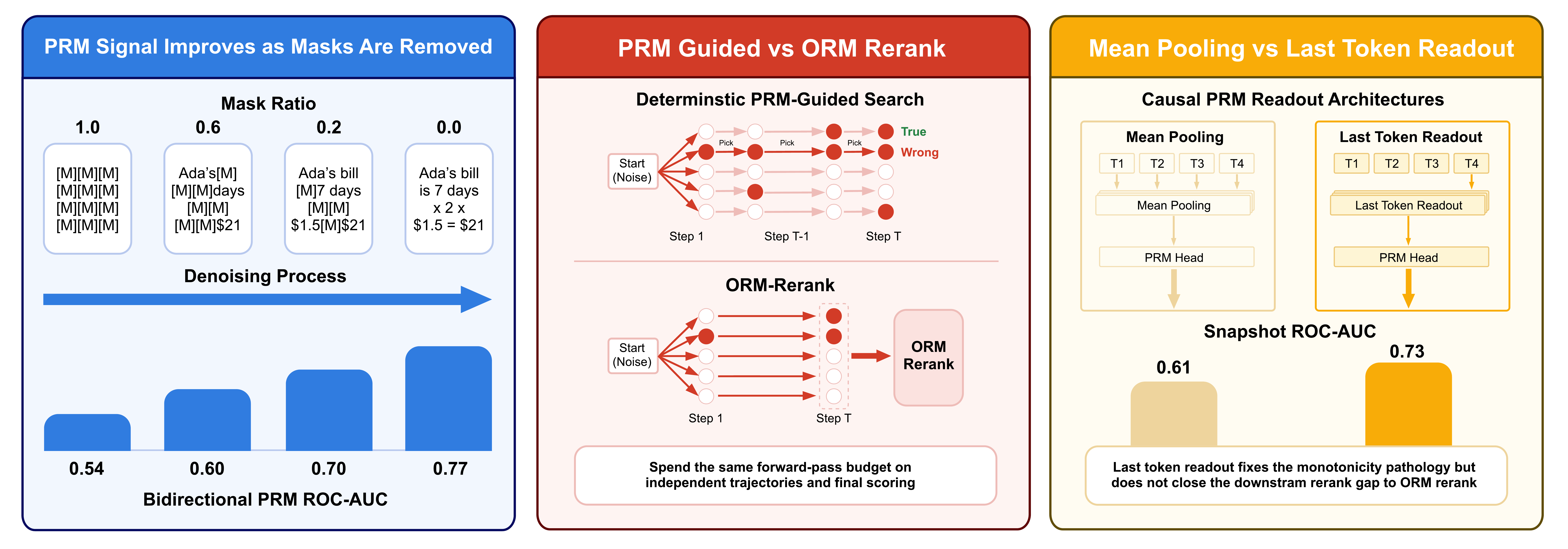}};
    \begin{scope}[x={(fw.south east)},y={(fw.north west)}]
    \fill[fwbg] (0.11987,0.12491) rectangle (0.14437,0.17433);
    \fill[fwbg] (0.19745,0.12491) rectangle (0.22196,0.17433);
    \end{scope}
    \end{tikzpicture}
    \caption{\textbf{Framework overview.} \textbf{(A) Signal decay.} Bidirectional PRM ROC-AUC rises from $0.54$ in the most-masked bucket to $0.77$ in the least-masked bucket $[0.0, 0.1)$ and $0.78$ on fully decoded states, so guidance decisions made at high mask ratios rely on the weakest signal. \textbf{(B) Diversity collapse.} Deterministic top-$1$ PRM pruning collapses the candidate pool, while ORM Rerank spends the same forward-pass budget on independent trajectories. \textbf{(C) Readout mismatch.} For causal PRM variants, switching from mean pooling to last-token pooling raises ROC-AUC on final states from $0.61$ to $0.73$ but does not close the downstream gap to ORM Rerank.}
    \label{fig:framework}
\end{figure}

Process reward models have been highly effective for AR mathematical reasoning, where
they score linear prefixes or stepwise traces and improve solution selection
\citep{lightman2023letsverify,wang2024mathshepherd}. But this success does not transfer
automatically to dLLMs. An AR PRM assumes that evidence arrives in prefix order; a dLLM
snapshot is instead a scattered subset of revealed positions, where later tokens may be
visible while earlier tokens remain hidden. This creates an architectural mismatch for
causal scorers built around sequential factorization and positional conventions
\citep{su2024roformer}, and motivates bidirectional scorers that attend to every visible
token. Recent dLLMs such as Dream, LLaDA, and Block Diffusion mainly improve the generator
or its sampling procedure \citep{ye2025dream,nie2025llada,arriola2025block}, rather than
asking whether reward guidance beats simple best-of-$N$, self-consistency
\citep{wang2023selfconsistency}, or outcome reward model (ORM) reranking baselines under
matched compute.
The natural recipe, PRM Guided, adapts PRM beam search from AR reasoning
\citep{snell2024scaling} to the denoising loop: it branches to $K$ candidates at fixed
intervals and keeps only the one the PRM scores highest (top-$1$). Compared with ORM
Rerank, which samples $N$ complete solutions independently and keeps the one the ORM
scores highest, it adds an intermediate state scorer, scorer calls inside the denoising
loop, and early pruning; if that machinery does not pay at matched compute, it is hard
to justify. Five questions remain open:
(i) does PRM Guided beat ORM Rerank at the same forward-pass budget?
(ii) does PRM signal remain useful at high mask ratios, when most tokens are still hidden?
(iii) is deterministic pruning safe for the diversity of the candidate pool?
(iv) can a scorer trained on intermediate states make the final choice? and
(v) is mean pooling the right readout for a causal PRM?

We answer them with a diagnostic protocol (Figure~\ref{fig:framework}) that charges
denoising, PRM scoring, and ORM scoring in the same forward-pass unit and separates the
damage guidance does to the candidate pool from the quality of the final selection.
\textbf{On Dream-7B, the answer to (i) is no on GSM8K, MATH, and MBPP}: on GSM8K, ORM
Rerank with 8 samples beats PRM Guided at every budget up to four times its compute. The gap has
two separable sources: guidance prunes on a weak signal, and on GSM8K and MATH the PRM is
a poor final judge. MBPP, where the PRM judges finished programs as well as the ORM,
isolates the cost of guidance itself.

\textbf{Contributions.}
\begin{itemize}[leftmargin=*,topsep=2pt,itemsep=2pt]
\item A matched forward-pass protocol for comparing guided and unguided dLLM reasoning at fixed inference budgets (\Cref{sec:setup}).
\item A matched-compute ranking that holds across math and code: with task-matched ORMs on Dream-7B, deterministic top-$1$ PRM guidance trails ORM Rerank by $9.95$ and $12.69$\,pp on GSM8K at $K{=}8$ and $K{=}32$, by $9.85$\,pp on MATH~\citep{hendrycks2021math}, and by $12.16$\,pp on MBPP~\citep{austin2021program} (\Cref{sec:empirical}).
\item A decomposition of the gap into pool damage and terminal selection: ROC-AUC decays from $0.77$ to $0.54$ with mask ratio, top-$1$ pruning cuts the Oracle ceiling by $13.75$\,pp, a matched SMC sampler that restores most of that ceiling still selects at the top-$1$ level, and a final-state PRM matches the ORM on the same candidates (\Cref{sec:empirical,sec:mechanism-decay,sec:mechanism-diversity,sec:mechanism-terminal}).
\item Bidirectional PRMs beat causal PRMs on Dream-7B and LLaDA-8B-Base, and a readout fix recovers most of the causal gap: last-token pooling raises final-state ROC-AUC from $0.61$ to $0.73$ and restores best-of-$N$ reranking gains (\Cref{sec:mechanism-causal}, App.~\ref{app:llada_prelim}).
\end{itemize}


\section{Related Work}
\label{sec:related}

\textbf{AR reward models and verifier reranking.}
Outcome verifiers and PRMs are established tools for AR mathematical reasoning: a model generates chain-of-thought solutions~\citep{wei2022cot}, a verifier scores final answers or stepwise traces, and the system selects or guides toward completions that score higher~\citep{cobbe2021gsm8k,lightman2023letsverify,uesato2022process,wang2024mathshepherd,luo2024omegaprm,shao2024deepseekmath,guo2025deepseekr1}. Recent verifier training reformulations~\citep{wang2025grpoverif} and turn-level reward designs for multi-turn reasoning~\citep{wei2025mtgrpo} continue this AR-centric line of work. These methods assume a linear prefix order: the scorer sees a context that grows from left to right and evaluates either the current prefix or the final completion. dLLM snapshots violate this assumption because their visible tokens form a scattered subset of positions rather than a prefix.

\textbf{dLLM reasoning and decoding.}
Diffusion language models and masked-token generators provide a different test-time substrate from AR decoding~\citep{hoogeboom2021argmax,austin2021d3pm,li2022diffusionlm,lou2024sedd,sahoo2024mdlm,gulrajani2023likelihood,chang2022maskgit}. Recent dLLMs such as Dream and LLaDA, and work on block diffusion, improved samplers, and parallel decoders, mainly target generation quality or decoding speed~\citep{ye2025dream,nie2025llada,arriola2025block,zheng2025timeagnostic,zhou2025dos,kim2025dapd,ringel2025demask}. Guidance in continuous diffusion relies on aligned continuous scores or gradients~\citep{dhariwal2021diffusion,ho2022classifierfree}, whereas masked-token reasoning requires a scorer that remains informative on partially observed discrete states; denoising pretraining with bidirectional encoders, as in BERT and BART~\citep{devlin2019bert,lewis2020bart}, motivates the bidirectional attention choice for PRMs that score snapshots.

\textbf{Test-time scaling and reward guidance for dLLMs.}
A growing line of work spends extra inference compute on dLLMs. Particle samplers resample denoising trajectories with sequential Monte Carlo (SMC)~\citep{ou2025discrete} or refine whole trajectories with particle Gibbs~\citep{dang2026trajectory}, and remasking samplers let the model revise tokens after they are unmasked~\citep{wang2025remdm}. Reward-free guidance avoids explicit PRMs, which it argues are hard to train on partially masked states, and derives an implicit process reward from a post-trained dLLM~\citep{chen2025rfg}. Correct answers often appear in the middle of denoising and are later overwritten~\citep{wang2026temporal}, and RL post-training with outcome rewards improves dLLM reasoning~\citep{zhao2025d1}. These works propose new samplers, guidance rules, or training objectives. We train explicit intermediate-state PRMs and measure where their guidance spends compute; our SMC control belongs to the particle family and shows that restoring the candidate pool does not repair the final choice when an intermediate-state scorer makes it.

\textbf{Matched compute evaluation gap.}
Best-of-$\Nsamples$, self-consistency, and ORM reranking are the natural baselines for spending extra test-time compute~\citep{wang2023selfconsistency,nakano2021webgpt,gao2023scaling,li2023makinglargelanguagemodels,snell2024scaling,brown2024monkeys}. Tree search methods such as Tree of Thoughts~\citep{yao2023tot} and tool-augmented reasoning agents~\citep{deepagent2025} also spend extra test-time compute, but they build on AR generation in prefix order rather than on scattered dLLM snapshots. Our diagnostic question is whether intermediate-state PRM guidance beats these baselines when denoising and scorer calls are charged under the same forward-pass budget.


\section{Matched Compute Diagnostic Protocol}
\label{sec:setup}

We compare reward-guided dLLM reasoning methods under a shared forward-pass budget. The accounting charges Vanilla sampling, Majority voting, ORM Rerank, PRM Guided, PRM Hybrid, and Oracle ceilings in the same unit: one forward pass through a dLLM-scale model for one candidate state. This removes hidden scorer cost as a confounder before we compare accuracy, candidate diversity, and scorer behavior.

\subsection{Notation and PRM architecture}

\textbf{Notation.}
Let $\seqlen$ denote sequence length, $\Tsteps{=}128$ the number of denoising steps, and $\bx_t$ a partially masked state at step $t$ with mask ratio $\mr{t}\in[0,1]$; $\mr{\Tsteps}{=}1$ is fully masked and $\mr{0}{=}0$ is fully decoded. We write $\Nsamples$ for the number of independent complete samples used by best-of-$N$ methods, $\Kcands$ for the branch width in PRM Guided, $\BranchEvery\in\{16,32,48,64\}$ for the denoising interval between PRM calls, and $\reward{t}$ for a scorer applied to state $\bx_t$.

\textbf{PRM definition and architecture.}
Here, PRM denotes an \emph{outcome-supervised intermediate-state value model}: each partial state inherits the final-correctness target of its completed trajectory. We study Dream-v0-Instruct-7B~\citep{ye2025dream} as the primary dLLM and LLaDA-8B-Base~\citep{nie2025llada} as the cross-backbone check. The PRM is a frozen dLLM backbone with trainable LoRA adapters~\citep{hu2022lora} plus a two-layer MLP reward head over pooled solution hidden states and a $256$-dimensional sinusoidal step-index embedding (App.~\ref{app:prm_training}). We compare bidirectional PRMs with full self-attention against causal PRMs with an L$\to$R attention mask, and isolate readout effects by retraining the causal scorer with last-token pooling in Section~\ref{sec:mechanism-causal}. Because the PRM is trained on states at every mask ratio, we call it the cross-mask PRM. Two scorers see only fully decoded states: the ORM, a bidirectional scorer trained on final states, and the final-state PRM, the bidirectional PRM retrained on final states (App.~\ref{app:additional_matched_controls}).

\subsection{Forward-pass compute accounting}
\label{sec:accounting}

A Vanilla sample requires $\Tsteps{=}128$ denoising passes. Majority$@N$ and Oracle$@N$ generate $N$ independent trajectories, so their denoising cost is $128N$. ORM Rerank$@N$ adds one ORM scoring pass per complete candidate, giving $C_{\mathrm{ORM}}(N){=}128N+N$.

\begin{figure}[t]
\centering
\begin{minipage}{0.97\linewidth}
\hrule
\vspace{2pt}
\textbf{Algorithm 1: PRM Guided segmental top-$1$ pruning on a dLLM}
\vspace{2pt}
\hrule
\vspace{4pt}
\textbf{Input:} dLLM $p_\theta$, PRM scorer $r_\phi$, prompt $q$, branching width $K$, interval $\BranchEvery$, denoising steps $T$.\\
\textbf{Output:} one generated sequence.
\vspace{2pt}
\hrule
\vspace{3pt}
\begin{enumerate}[leftmargin=*,itemsep=1pt,topsep=2pt]
\item Initialize $\bx_T \gets \mathrm{MASK}^L$ (fully masked with prompt).
\item Set $\mathrm{segments} \gets \lceil T / \BranchEvery \rceil$; $t_s \gets T$. The last segment is shorter when $\BranchEvery$ does not divide $T$.
\item \textbf{for} $s = 1, \dots, \mathrm{segments}$ \textbf{do}
\begin{enumerate}[label=\arabic*.,leftmargin=*,itemsep=1pt,topsep=1pt]
\item Replicate $\bx_{t_s}$ into $K$ copies $\{\bx^{(k)}_{t_s}\}_{k=1}^K$.
\item Run $\BranchEvery$ denoising steps on all $K$ copies \emph{in parallel} with independent per-copy token sampling, yielding $\{\bx^{(k)}_{t_s - \BranchEvery}\}_{k=1}^K$.
\item Score all $K$: $v_k \gets r_\phi(\bx^{(k)}_{t_s - \BranchEvery})$.
\item Prune to top-$1$: $\bx_{t_s - \BranchEvery} \gets \bx^{(\arg\max_k v_k)}_{t_s - \BranchEvery}$; update $t_s \gets t_s - \BranchEvery$.
\end{enumerate}
\item \textbf{return} $\bx_0$ (the retained trajectory after the final scored prune).
\end{enumerate}
\vspace{2pt}
\hrule
\end{minipage}
\end{figure}

\textbf{PRM Guided is charged for both denoising and every PRM scoring call, including the final segment.}
The segmental top-$1$ algorithm branches to $K$ candidates every $\BranchEvery$ steps, denoises all $K$ candidates for that segment, scores all $K$ candidates, and retains the highest scoring state; Algorithm~1 gives the exact procedure. With $\lceil T/\BranchEvery \rceil$ segments, the per-sample forward-pass cost is
\[
C_{\mathrm{PRM}}(K,\BranchEvery) \;=\; KT \;+\; K\,\bigl\lceil T/\BranchEvery \bigr\rceil .
\]
In the headline setting, $\BranchEvery{=}64$ and $T{=}128$ give $\lceil T/\BranchEvery \rceil{=}2$ segments and $130K$ passes: $1{,}040$ for $K{=}8$ and $4{,}160$ for $K{=}32$. The matched ORM Rerank budgets are $1{,}032$ and $4{,}128$ passes, so the comparison is matched within $\sim\!0.8\%$ and the small asymmetry favors PRM Guided; App.~\ref{app:compute} reports wall-clock validation.

\subsection{Diagnostic measurements}

Each question in \Cref{sec:intro} has a direct measurement. Accuracy at matched compute answers (i) (\Cref{sec:empirical}); PRM ROC-AUC across mask ratios answers (ii) (\Cref{sec:mechanism-decay}); answer diversity and Oracle@$K$ of the guided pool answer (iii) (\Cref{sec:mechanism-diversity}); reranking one shared candidate pool with different scorers, and PRM selection after an SMC sampler restores the pool, answer (iv) (\Cref{sec:empirical,sec:mechanism-terminal}); and readout ablations on causal PRMs answer (v) (\Cref{sec:mechanism-causal}).

\subsection{Training data and evaluation}

\textbf{PRM training uses on-policy dLLM states from the training split, not test prompts.}
The PRM is trained on on-policy intermediate states from Dream-7B denoising trajectories on GSM8K training problems, with binary final-correctness labels. Training and test trajectories share the sampler and the snapshot schedule, so their mask-ratio distributions match by construction. All training, tuning, and early stopping use the GSM8K \emph{train} split, with validation on held-out training problems.

\textbf{The primary evaluation is GSM8K test with strict answer extraction.}
GSM8K~\citep{cobbe2021gsm8k} has $1{,}319$ test problems; the best-of-$32$ test pool contains \NTestTrajectories{} complete trajectories. Vanilla, Majority, ORM Rerank, PRM Guided, and PRM Hybrid use the shared Dream sampling configuration: \texttt{temperature=0.5}, \texttt{alg\_temp=0.5}, \texttt{top\_p=1.0}, and $T{=}128$. Answers are scored with a strict regex extractor throughout, which avoids the inflated accuracy of lm-eval style last-number matching~\citep{gao2024lmeval} (App.~\ref{app:extraction}). MATH500 also serves as an out-of-distribution test for the GSM8K-trained scorers (App.~\ref{app:math500_full}); the task-specific MATH and MBPP controls train their own verifiers (App.~\ref{app:additional_matched_controls}).

\textbf{The baselines separate sampling, scoring, guidance, and oracle headroom.}
Vanilla runs one dLLM trajectory. Majority$@N$ picks the most frequent extracted answer among $N$ independent trajectories. ORM Rerank$@N$ scores the same independent final state candidates with an ORM trained on GSM8K and selects the highest scoring candidate. PRM Guided follows the segmental procedure above, and PRM Hybrid keeps all $K$ candidates at its final segment to measure diversity and post hoc selector ceilings. Oracle$@N$/Oracle@$K$ marks a problem correct if any pool candidate is correct, used only as a verifier headroom ceiling.
The ORM comparison targets the setting where a task-matched final-answer verifier can be fit from the train split.


\section{Task-Matched ORM Reranking Beats Deterministic PRM Guidance at Matched Compute}
\label{sec:empirical}

We first test whether deterministic PRM Guided beats independent sampling plus ORM Rerank at the same forward-pass budget. On Dream-7B, ORM Rerank wins on GSM8K at both budgets and on task-specific MATH and MBPP. Section~\ref{sec:mechanism} decomposes the gap into pool damage from guidance and terminal selection quality.

\begin{figure}[t]
    \centering
    \includegraphics[width=\textwidth]{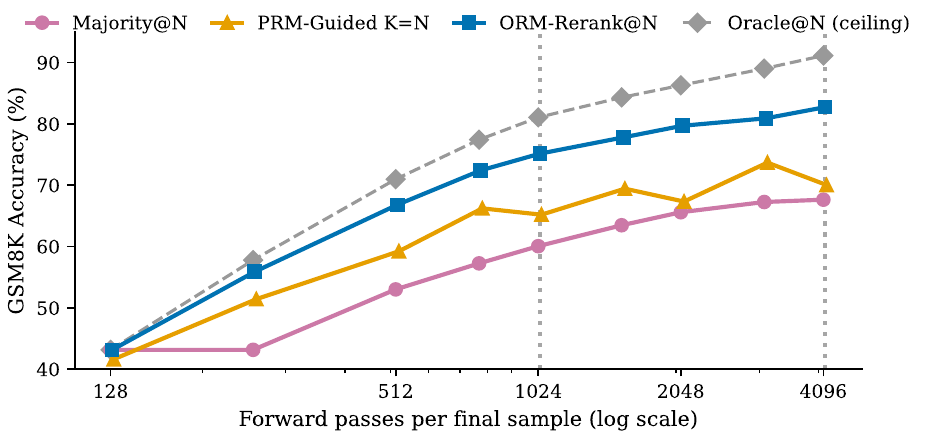}
    \caption{\textbf{At matched compute, reranking with a task-matched ORM outperforms PRM Guided on GSM8K with Dream-7B.} ORM Rerank lies above deterministic PRM Guided across the tested compute range, and ORM Rerank@$8$ exceeds every PRM Guided budget up to $K{=}32$. Oracle@$N$ marks the perfect-selector ceiling over independent samples, showing remaining verifier headroom.}
    \label{fig:hero_pareto}
\end{figure}

At the headline budget, ORM Rerank$@8$ reaches $75.13\%$ accuracy while PRM Guided $K{=}8$ reaches $65.18\%$; at the larger budget, ORM Rerank$@32$ reaches $82.71\%$ while PRM Guided $K{=}32$ reaches $70.02\%$ (Table~\ref{tab:pareto}, Figure~\ref{fig:hero_pareto}).
The gap is $9.95$\,pp at $K{=}8$ and $12.69$\,pp at $K{=}32$.
PRM Guided still beats Majority@$N$ at both budgets, so the PRM carries real signal; the same compute buys better final answers when spent on independent samples plus a final-state ORM. On the same samples, ORM Rerank beats Majority@$N$ by $15.1$\,pp at both budgets, so the final-state verifier, not sampling alone, drives the gain.
Measured wall-clock ratios agree with forward-pass predictions within $9\%$, and the ordering survives wall-clock matching: ORM Rerank$@6$ reaches $72.40\%$ in about $126$\,s per problem, against $65.18\%$ in $156$\,s for PRM Guided $K{=}8$ (App.~\ref{app:compute}).

ORM Rerank$@8$ beats every PRM Guided budget we ran, including $K{=}32$, which spends four times its compute and reaches $70.02\%$; the best PRM Guided result, $73.67\%$ at $K{=}24$, still trails it. Across the full sweep, ORM Rerank improves at every budget from $N{=}2$ to $N{=}32$, while PRM Guided is not monotone in $K$ (App.~\ref{app:full_pareto}). Even ORM Rerank$@32$ sits $8.4$\,pp below the $91.13\%$ of Oracle$@32$, so better final verifiers still have room to grow. The advantage does not hinge on the argmax rule: verifier-weighted voting~\citep{li2023makinglargelanguagemodels} stays within $0.4$\,pp of ORM Rerank at every budget.

A stronger final selector does not rescue the candidate pool: even a perfect selector over the PRM Hybrid $K{=}8$ pool reaches only $67.30\%$, $7.83$\,pp below ORM Rerank$@8$ (Section~\ref{sec:mechanism-diversity}). The top-$M$ and SMC controls in Sections~\ref{sec:mechanism-diversity} and~\ref{sec:mechanism-terminal} measure how much of this loss is recovered by keeping more candidates.

\begin{table}[t]
\centering
\small
\caption{Matched-compute GSM8K accuracy with Dream-v0-Instruct-7B. ORM Rerank leads at both budgets. PRM Guided entries average three independent runs at $K{=}8$ and two at $K{=}32$ (App.~\ref{app:full_pareto}). The diagnostic rows isolate causal PRM readout. Bold: best comparable value in each accuracy column.}
\label{tab:pareto}
\resizebox{0.82\linewidth}{!}{%
\begin{tabular}{l c c cc}
\toprule
\rowcolor{ArxivTableHead}
 & & & \multicolumn{2}{c}{\tablehead{GSM8K Accuracy (\%)}} \\
\cmidrule{1-5}
\rowcolor{ArxivTableHead}
\tablehead{Method} & \tableheadmath{N/K} & \tablehead{Forward passes} & \tableheadmath{N{=}8 / K{=}8} & \tableheadmath{N{=}32 / K{=}32} \\
\midrule
\rowcolor{ArxivGroupNeutral}
\multicolumn{5}{l}{\tablehead{Matched-compute comparison}} \\
\midrule
Vanilla & 1 & 128 & \multicolumn{2}{c}{43.14} \\
Majority@$N$ & $N$ & $N{\times}128$ & 60.05 & 67.63 \\
PRM Guided ($\BranchEvery{=}64$) & $K$ & $130K$ & 65.18 & 70.02 \\
\rowcolor{ArxivTableRow}
ORM Rerank@$N$ & $N$ & $N{\times}128+N$ & \best{75.13} & \best{82.71} \\
Oracle@$N$ & $N$ & $N{\times}128$ & \emph{81.05} & \emph{91.13} \\
\midrule
\rowcolor{ArxivGroupNeutral}
\multicolumn{5}{l}{\tablehead{Causal PRM readout diagnostic}} \\
\midrule
Causal PRM Rerank, mean pool & $N$ & $N{\times}128+N$ & 43.90 & 40.56 \\
Causal PRM Rerank, last token & $N$ & $N{\times}128+N$ & 49.66 & \best{50.64} \\
\bottomrule
\end{tabular}}
\end{table}


\begin{table}[t]
\centering
\small
\caption{Controls beyond the GSM8K headline. \emph{Top:} task-specific verifiers on held-out MATH and MBPP problems. \emph{Bottom:} terminal scoring on the same GSM8K candidate pool. Splits and full results are in App.~\ref{app:additional_matched_controls}. Bold: best comparable value in each result column.}
\label{tab:additional_matched_main}
\resizebox{0.72\linewidth}{!}{%
\begin{tabular}{lccc}
\toprule
\rowcolor{ArxivTableHead}
\tablehead{Task-specific verifiers} & \tablehead{ORM Rerank} & \tablehead{PRM Rerank} & \tablehead{PRM Guided} \\
\midrule
MATH, $N{=}K{=}8$ & \best{30.65} & 20.80 & 20.80 \\
MBPP & 63.04 & \best{65.47} & 50.88 \\
\midrule
\rowcolor{ArxivGroupNeutral}
\tablehead{GSM8K, same candidate pool} & \tablehead{ORM Rerank} & \tablehead{Final-state PRM} & \tablehead{Cross-mask PRM} \\
\midrule
$N{=}8$ & 75.13 & \best{75.40} & 42.84 \\
$N{=}32$ & 82.71 & \best{82.79} & 65.35 \\
\bottomrule
\end{tabular}}
\end{table}

The ordering carries over to other tasks once the verifier is trained for the task (Table~\ref{tab:additional_matched_main}); a GSM8K-trained ORM does not transfer to MATH500 (App.~\ref{app:math500_full}). On MATH~\citep{hendrycks2021math} at $N{=}K{=}8$, ORM Rerank reaches $30.65\%$ against $20.80\%$ for both PRM Rerank and deterministic PRM Guided; the ORM lead holds on identical stored candidates, by $11.10$\,pp (App.~\ref{app:additional_matched_controls}). On MBPP~\citep{austin2021program}, ORM Rerank reaches $63.04\%$ against $50.88\%$ for PRM Guided. The MBPP PRM already reaches $65.47\%$ when reranking final programs, on par with the ORM, so the $14.6$\,pp it loses under guidance is the cost of guidance itself.

Same-pool reranking separates terminal scoring from pruning. On the GSM8K candidate pool, the final-state PRM matches ORM Rerank at both budgets, while the cross-mask PRM trails it by $32.3$ and $17.4$\,pp (Table~\ref{tab:additional_matched_main} and App.~\ref{app:scorer_breakdown}). At $N{=}8$ the cross-mask PRM reranks no better than random selection (App.~\ref{app:sanity}), even though its pooled ROC-AUC on final states is $0.78$: pooled discrimination does not imply within-problem ranking (Proposition~\ref{prop:pooled_main}). Its signal shows up in larger pools, where it beats random selection by $8.6$\,pp at $N{=}16$ and $22.1$\,pp at $N{=}32$. The PRM is a capable final selector once trained on final states; the gap comes from cross-mask training and, under guidance, from pruning.


\section{Why Deterministic PRM Guidance Underperforms}
\label{sec:mechanism}

The matched-compute gap has two parts. The first is pool damage from guidance: mask-ratio signal decay and deterministic pruning remove correct candidates before any final selector sees them. The second is terminal selection: trained across all mask ratios, the PRM separates correct from incorrect final solutions with ROC-AUC $0.78$, against $0.96$ for the ORM trained on final states alone. We then analyze causal readout mismatch as a separate diagnostic for causal PRM variants and test whether right context explains the bidirectional advantage. Table~\ref{tab:decomposition} shows both failures at the headline budget.

\begin{table}[t]
\centering
\small
\caption{Pool quality and final selection on GSM8K at the headline budget ($K{=}N{=}8$). \emph{Top:} the first two columns measure the candidate pool and the last three the accuracy of the final choice from it. Top-$1$ guidance damages the pool; SMC repairs most of it, yet the cross-mask PRM still selects at the top-$1$ level, while final-state scorers select well from independent samples. \emph{Bottom:} how often a PRM cut removes every lineage that later reaches a correct answer, by the stored state at which the cut is applied. Dashes mark configurations that were not run. Full results are in App.~\ref{app:additional_matched_controls}. Bold: best comparable value in each accuracy column.}
\label{tab:decomposition}
\resizebox{0.75\linewidth}{!}{%
\begin{tabular}{lccc>{\columncolor{ArxivTableCol}}cc}
\toprule
\rowcolor{ArxivTableHead}
 & \multicolumn{2}{c}{\tablehead{Candidate pool}} & \multicolumn{3}{c}{\tablehead{Accuracy of the final choice (\%)}} \\
\cmidrule(lr){2-3}\cmidrule(lr){4-6}
\rowcolor{ArxivTableHead}
\tablehead{Search} & \tablehead{Oracle@8 (\%)} & \tablehead{Unique answers} & \tablehead{Cross-mask PRM} & \tablehead{ORM} & \tablehead{Final-state PRM} \\
\midrule
Independent samples & 81.05 & 4.31 & 42.84 & \cellcolor{ArxivTableMix}\best{75.13} & \best{75.40} \\
Top-$1$ guidance & 67.30 & 1.75 & 65.18 & -- & -- \\
SMC & 77.89 & 3.95 & 65.48 & -- & -- \\
\midrule
\rowcolor{ArxivGroupNeutral}
\tablehead{Removal risk (\%)} & \tablehead{Initial state} & \tablehead{Middle state} & \tablehead{Final state} & & \\
\midrule
Top-$1$ cut & 46.06 & 37.02 & 19.92 & & \\
Top-$2$ cut & -- & 22.54 & 13.01 & & \\
Top-$4$ cut & -- & 9.94 & 5.93 & & \\
\bottomrule
\end{tabular}}
\end{table}

\begin{figure*}[t]
    \centering
    \includegraphics[width=\textwidth]{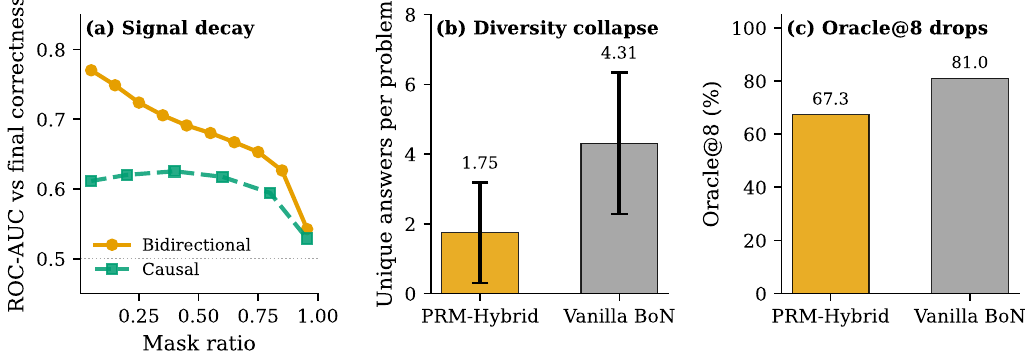}
    \caption{\textbf{Mechanisms behind the matched-compute gap.}
    \textbf{(a)} PRM discrimination degrades as mask ratio increases.
    \textbf{(b)} Deterministic top-$1$ pruning reduces answer diversity; error bars show the standard deviation across problems, and Vanilla BoN denotes best-of-$8$ independent samples.
    \textbf{(c)} The reachable Oracle ceiling falls, showing that guidance prunes away correct candidates that independent sampling would have retained.}
    \label{fig:mechanism}
\end{figure*}

\subsection{Mask-Ratio Signal Decay}
\label{sec:mechanism-decay}

Pool damage starts with a weak early signal.
Bidirectional PRM ROC-AUC falls monotonically from $0.77$ on nearly decoded states to $0.54$ on almost fully masked ones (Figure~\ref{fig:mechanism}(a)). No scorer can avoid this trend once the state stops carrying information about the outcome:

\begin{proposition}[informal]
\label{prop:mi_main}
For any scorer, ROC-AUC on states at a given mask ratio is at most $\tfrac12+\sqrt{I(\bx_t;y)/(2\pi_0\pi_1)}$, where $I(\bx_t;y)$ is the mutual information between the state $\bx_t$ and final correctness $y$, and $\pi_0,\pi_1$ are the class priors.
\end{proposition}

As masking removes that information, the bound falls toward chance, so decay with mask ratio is expected for any scorer, not only ours (Proposition~\ref{prop:mi_auc} in App.~\ref{app:theory}).

The causal PRM is weak across all mask-ratio buckets, not only at early denoising stages.
Trained on identical data, it stays between $0.60$ and $0.63$ through low and middle mask ratios and remains below the bidirectional PRM in every bucket; \Cref{sec:mechanism-causal} shows that much of the final-state gap comes from mean-pool readout rather than the attention mask itself.

The decay is not an artifact of inherited binary labels. Relabeling $10{,}000$ training states with eight fresh rollouts each and retraining under an otherwise identical setup raises pooled ROC-AUC by only $0.011$ and moves GSM8K accuracy by $0.61$\,pp, small next to the $10$\,pp gap to ORM Rerank. Scored against fresh rollouts instead, the decay stays intact, even though the two label sources disagree on $21\%$ of states in the most-masked bucket, against $0.5\%$ in the least-masked one (App.~\ref{app:additional_matched_controls}).

\subsection{Diversity Collapse Under Deterministic Pruning}
\label{sec:mechanism-diversity}

Deterministic pruning then collapses answer diversity, an effect known from beam search~\citep{holtzman2020curious,eikema2020map} that is severe in dLLM guidance.
At $K{=}N{=}8$, PRM Hybrid keeps only $1.75$ unique answers per problem against $4.31$ for independent samples, and answer entropy falls fourfold.

The diversity collapse lowers the best possible accuracy of the guided candidate pool.
Oracle$@8$ drops from $81.05\%$ for independent samples to $67.30\%$ for PRM Hybrid, a $13.75$\,pp ceiling loss: PRM Guided prunes away correct candidates that simple sampling keeps.

Keeping more candidates and guiding later both help, but neither closes the gap.
Top-$M{=}2$ reaches $69.70\%$, $4.5$\,pp above top-$1$ and still $5.4$\,pp below ORM Rerank@$8$; in single-run sweeps, accuracy rises at every step from $59.4\%$ at $\BranchEvery{=}16$ to $66.5\%$ at $\BranchEvery{=}64$, even though the most frequent guidance also costs the most, $1{,}088$ passes against $1{,}040$, consistent with querying the PRM where its signal is stronger (Apps.~\ref{app:topm} and~\ref{app:branching}).

An offline counterfactual over the stored trajectories locates the pool damage early in denoising. A top-$1$ cut removes every lineage that eventually reaches a correct answer in $46\%$ of cases at the initial stored state and in $20\%$ at the final one. Wider cuts shrink the risk: keeping four candidates brings it to $10\%$ at the middle state and $6\%$ at the final one (Table~\ref{tab:decomposition}).

\subsection{The PRM as a Final Judge}
\label{sec:mechanism-terminal}

Repairing the pool does not repair the final choice.
An SMC sampler~\citep{del2006smc} that resamples $K{=}8$ particles by PRM score at the same $1{,}040$ passes (App.~\ref{app:algo}) restores most of the pool: Oracle@$8$ rises from $67.30\%$ to $77.89\%$, recovering $10.59$ of the $13.75$ points that top-$1$ pruning loses against independent sampling. Its PRM-selected accuracy stays at $65.48\%$ with weighted answer voting and $66.34\%$ with the top-scoring particle, level with top-$1$ guidance and well below the $75.13\%$ of ORM Rerank@$8$ (Table~\ref{tab:decomposition} and App.~\ref{app:additional_matched_controls}). With the pool repaired, the gap sits in the final choice, which the cross-mask PRM makes less reliably than a final-state verifier (\Cref{sec:empirical}).

Good pooled discrimination does not protect the final choice:

\begin{proposition}[informal]
\label{prop:pooled_main}
For every $N\ge 2$ and $\epsilon\in(0,1)$, some scorer has pooled ROC-AUC $1-\epsilon$ while its top-$1$ choice among the $N$ candidates of each problem is no better than a uniform random pick.
\end{proposition}

A scorer that tracks problem difficulty rather than candidate quality achieves this: it separates easy problems from hard ones across the pool but gives every candidate of a problem the same score (Proposition~\ref{prop:pooled_auc_rerank} in App.~\ref{app:theory}). The cross-mask PRM is not blind within problems. On the $1{,}182$ test problems that contain both correct and incorrect candidates, its scores correlate with correctness, with median Kendall $\tau$ of $0.44$, and are higher on average for correct candidates in $88\%$ of them; a moderate pairwise edge of this kind need not survive a top-$1$ pick among many candidates (Corollary~\ref{cor:kendall_top1}).

\subsection{Causal PRM Readout Mismatch}
\label{sec:mechanism-causal}

Most of the causal PRM's weakness comes from mean pooling, not from causal attention. On final states, the mean-pooled causal PRM reaches ROC-AUC $0.61$ against $0.78$ for the bidirectional PRM. Last-token pooling, standard for AR reward models~\citep{ouyang2022instructgpt,stiennon2020summarize}, raises it to $0.73$ and closes about $70\%$ of the gap. A smaller bidirectional advantage survives both longer training and the ORM protocol. Doubling training from $15$K to $31$K steps narrows the classification-accuracy gap in the least-masked bucket from $13.6$ to $9.5$\,pp, and the bidirectional PRM stays ahead in all ten mask buckets; trained on final states only, the causal scorer still trails by $7.8$\,pp (Apps.~\ref{app:one_epoch} and~\ref{app:orm_protocol_controls}). Mean pooling hurts only the causal scorer: under the ORM protocol, the bidirectional scorer reaches $91.83\%$ accuracy with last-token pooling and $91.78\%$ with mean pooling.

\begin{figure}[t]
    \centering
    \includegraphics[width=0.75\columnwidth]{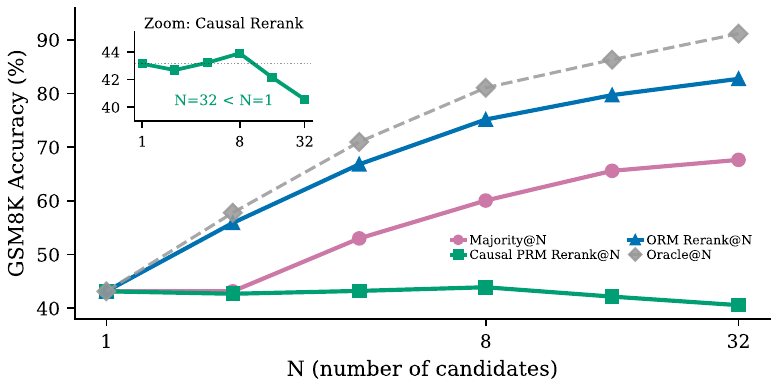}
    \caption{\textbf{Mean-pooled causal PRMs fail as best-of-$N$ rerankers.}
    Majority and ORM Rerank improve with more candidates, while the mean-pooled causal PRM remains near the single-sample baseline and dips at high $N$. The inset highlights the nonmonotonic region.}
    \label{fig:causal_rerank}
\end{figure}

The readout also explains the reranking failure in Figure~\ref{fig:causal_rerank}: at $N{=}32$ the mean-pooled causal reranker falls to $40.56\%$, below a single sample, while the last-token variant rises to $50.64\%$ (App.~\ref{app:lasttoken_pareto}). Both remain far below the $82.71\%$ of ORM Rerank@$32$, as does the bidirectional PRM at $65.35\%$: the terminal selection gap of \Cref{sec:empirical} (App.~\ref{app:scorer_breakdown}).

\begin{table}[t]
\centering
\small
\caption{One GSM8K test problem (id $526$), the one sketched in Figure~\ref{fig:framework}(A): Ada uses $12$\,kWh a day, adds a device that uses $2$\,kWh a day, and pays $\$1.50$ per kWh; the question asks for the weekly bill difference, which is $\$21$. Three of the independent candidates are shown with their causal PRM scores.}
\label{tab:case526}
\begin{tabular}{llcc}
\toprule

\rowcolor{ArxivTableHead}
\tablehead{Candidate} & \tablehead{Final computation} & \tablehead{Causal PRM score} & \tablehead{Correct} \\
\midrule
\mdseries
$t{=}13$ & $2 \times \$1.50 = \$30$ & $+1.001$ & $\times$ \\
$t{=}25$ & $(12-2) \times \$1.50 = \$15$ & $-0.827$ & $\times$ \\
$t{=}3$ & $2 \times \$1.50 \times 7 = \$21$ & $-0.908$ & $\checkmark$ \\
\bottomrule
\end{tabular}
\end{table}

Table~\ref{tab:case526} shows the failure on a single problem: the causal PRM ranks a wrong trajectory first and the correct one last of the three, while the ORM selects the correct one. The same inversion, a confident preference for a wrong trajectory while a correct one is available, occurs on $639$ of the $1{,}319$ test problems, or $48\%$ (App.~\ref{sec:case_study}).

\begin{figure}[t]
    \centering
    \includegraphics[width=0.55\textwidth]{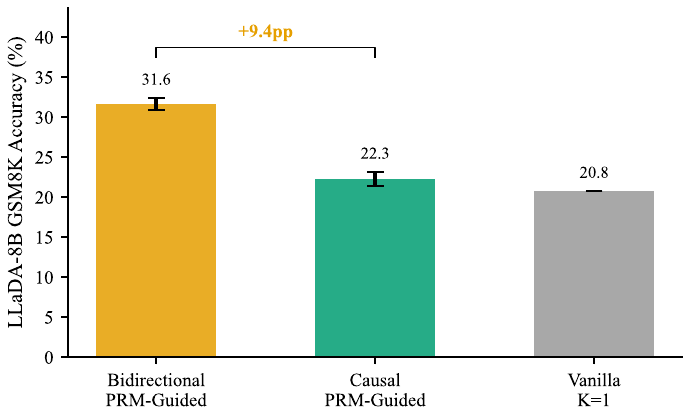}
    \caption{\textbf{Bidirectional PRMs beat causal PRMs on LLaDA-8B-Base.} Mean GSM8K accuracy of PRM Guided ($K{=}8$) over all tested configurations on the full test set, with the single-trajectory Vanilla baseline; error bars show standard deviations across configurations.}
    \label{fig:llada_prelim}
\end{figure}

LLaDA reproduces the bidirectional advantage on a second dLLM backbone: averaged over all tested configurations on the full GSM8K test set, bidirectional PRM guidance reaches $31.64\%$ against $22.25\%$ for causal PRM guidance (Figure~\ref{fig:llada_prelim} and App.~\ref{app:llada_prelim}). Over the $20.77\%$ Vanilla baseline, bidirectional guidance adds $10.87$\,pp and causal guidance only $1.48$\,pp.

\subsection{Right Context Alone Does Not Explain the Advantage}
\label{sec:mechanism-ruleout}

Could the bidirectional PRM win simply by reading tokens to the right of each position? Two controls do not support this. Zeroing the right half of each snapshot changes the bidirectional PRM's ROC-AUC by about $0.001$, and averaging a causal PRM over forward and reversed inputs adds nothing and leaves it $0.084$ below the bidirectional PRM, because the reversed pass scores at chance (ROC-AUC $0.500$) on a RoPE backbone~\citep{su2024roformer} that never saw reversed text (App.~\ref{app:m5_caveats}). Three explanations remain open: joint modeling of unordered positions, readout design, and a shift in training distribution. The first is the most actionable: bidirectional attention lets every position condition on whichever tokens are revealed, while a causal scorer reads a scattered partial state in one fixed left-to-right order. A permutation-LM scorer~\citep{yang2019xlnet}, trained over arbitrary orderings yet autoregressive at inference, would test it (Apps.~\ref{app:m4m5_figure} and~\ref{app:hypothesis_family}).


\section{Discussion}
\label{sec:discussion}

\textbf{Recommendations.} The results translate into five practices for reward-guided dLLM decoding. Charge intermediate-state scorers and final-state verifiers in the same forward-pass unit, or hidden scorer cost confounds the comparison. Report Oracle@$K$ next to accuracy: it exposes candidate-pool collapse independently of the final selector. Guide late and keep more than one candidate: sparse branching and top-$2$ retention both recover accuracy, and keeping four candidates cuts the risk of losing every correct lineage at the final state from $20\%$ to $6\%$ (\Cref{sec:mechanism-diversity}). Hand the final choice to a verifier trained on final states, because pooled scorer quality does not guarantee within-problem ranking: the cross-mask PRM, with pooled ROC-AUC $0.78$ on final states, reranks no better than random selection at $N{=}8$ (Proposition~\ref{prop:pooled_main}). For causal PRMs, fix the readout first: last-token pooling recovers about $70\%$ of the ROC-AUC gap to the bidirectional PRM. These checks are not specific to diffusion: any search that prunes on intermediate scores and selects with the same scorer, including PRM beam search for AR models~\citep{snell2024scaling}, splits into pool damage and selection error the same way.

\textbf{Threats to validity.} The central result uses Dream-7B and the deterministic top-$1$ PRM Guided recipe, with GSM8K as the primary task and MATH and MBPP as task-specific controls. The SMC control covers one diversity-preserving sampler; other stochastic decoders can be benchmarked with the released toolkit. The ORM advantage requires a verifier trained for the task: a GSM8K-trained ORM does not transfer to MATH500. LLaDA reproduces the bidirectional advantage; an LLaDA-specific ORM comparison and the permutation-LM test are natural next steps.

\textbf{Conclusion.} dLLM intermediate states carry usable reward signal, but deterministic guidance spends it in the wrong place: it prunes early, where the scorer is least reliable, and on GSM8K and MATH it leaves the final choice to a weaker intermediate-state scorer. Under matched forward-pass compute, independent sampling with a task-matched ORM beats deterministic PRM guidance on GSM8K, MATH, and MBPP; on MBPP the PRM already judges finished programs as well as the ORM, so the loss there is the cost of guidance itself. The decomposition gives dLLM guidance two concrete targets: keep correct partial solutions alive through early denoising, and leave the final choice to a verifier trained on final states. We release the snapshot corpus and evaluation toolkit to measure both under the same protocol.

{\small\bibliographystyle{acl_natbib}\bibliography{references}}
\newpage
\appendix
\section{Evaluation Protocol and Implementation Details}
\label{app:protocol_details}

\subsection{PRM training details}
\label{app:prm_training}

The PRM adds trainable LoRA adapters~\citep{hu2022lora} with $r{=}16$ and $\alpha{=}32$ on \texttt{q\_proj} and \texttt{v\_proj} to the frozen dLLM backbone. Main comparison PRMs train with BCE loss for $2{,}000$ steps at batch size $32$, with a cosine learning rate schedule from $2{\times}10^{-5}$ to $0$.

\subsection{Answer-extraction sensitivity}
\label{app:extraction}

The main text uses a strict regex extractor (\texttt{extract\_gsm8k\_answer}) that checks $\#\#\#\#$, ``answer is:'', $\backslash\texttt{boxed}\{\}$, and a last-integer fallback, in that order.
The \texttt{lm-eval-harness}~\citep{gao2024lmeval} \texttt{flexible-extract} regex for GSM8K instead takes the last number in the output, which raises Vanilla accuracy on Dream-7B from $43.14\%$ (strict) to $54.28\%$ (flexible).
The $+11.14$\,pp offset comes from extraction alone, and last-number matching breaks on the distractor numerals common in dLLM outputs, so we use strict extraction throughout.

\subsection{PRM Guided algorithm}
\label{app:algo}

Algorithm~1 in Section~\ref{sec:accounting} specifies the PRM Guided procedure evaluated throughout the paper. This section gives its forward-pass cost and the variants used in the controls.

The per-sample forward-pass cost decomposes as $KT$ denoising passes plus $\mathrm{segments} \times K$ PRM scoring passes, with $\mathrm{segments}{=}\lceil T/\BranchEvery\rceil$. For the headline setting $T{=}128$, $\BranchEvery{=}64$: $\mathrm{segments}{=}2$, giving $128K + 2K = 130K$ total forward passes per final sample. With $K{=}8$ this is $1{,}040$, and with $K{=}32$ it is $4{,}160$; App.~\ref{app:compute} compares both with ORM Rerank. For the non-divisible case $\BranchEvery{=}48$ used in Apps.~\ref{app:branching} and~\ref{app:llada_prelim}, the three segments have $48$, $48$, and $32$ denoising steps, giving $128K + 3K = 131K$ passes, $0.8\%$ more than at $\BranchEvery{=}64$.

\paragraph{PRM Hybrid variant.}
Identical to Algorithm~1 except that the \emph{final} segment skips the top-$1$ prune and returns all $K$ candidates, from which any downstream selector, such as majority voting, the ORM, or the Oracle, picks one. \Cref{sec:mechanism-diversity} uses PRM Hybrid to measure Oracle ceilings and candidate-pool diversity.

\paragraph{Top-$M$ retention variant.}
This generalization relaxes top-$1$ pruning: at each segment it keeps the $M$ highest-scoring candidates, and each retained state spawns $K/M$ children at the next segment, so the total branch width stays $K$. For $M{=}1$ this reduces to Algorithm~1; for $M{=}K$ it coincides with running $K$ parallel chains without pruning. App.~\ref{app:topm} evaluates $M \in \{1,2\}$ at $K{=}8$.

\paragraph{ESS-tempered SMC variant.}
$K$ particles are weighted by tempered PRM scores after $64$ denoising steps and resampled systematically when the effective sample size (ESS) falls below a fixed threshold~\citep{del2006smc}; the final answer comes from either the weighted answer cluster or the highest-scoring particle after all $128$ steps. The sampler spends the same $1{,}040$ forward passes as PRM Guided at $K{=}8$, and App.~\ref{app:smc} reports its results.

\paragraph{Alternative guidance algorithms.}
Further alternatives not evaluated here include
energy-tilted sampling that blends the dLLM transition $p_\theta(\bx_{t-1}\mid\bx_t)$ with a PRM-weighted distribution $\propto p_\theta \cdot \exp(\beta r_\phi)$ via an annealing schedule,
and adaptive branching schedules that condition $\BranchEvery$ on the current mask ratio, consulting the PRM only where the ROC-AUC curve of \Cref{sec:mechanism-decay} is high.
The decomposition in \Cref{sec:mechanism} points to pairing diversity-preserving search with a final-state verifier as the terminal selector.

\subsection{Compute accounting}
\label{app:compute}

Per-sample forward-pass counts in Table~\ref{tab:pareto} decompose as follows.
A single denoising trajectory requires $T{=}128$ forward passes through the backbone.
Majority$@N$ and ORM Rerank$@N$ generate $N$ independent trajectories at $N \times 128$ denoising passes, and ORM Rerank adds $N$ ORM scoring passes.
PRM Guided with $K$ candidates and $\BranchEvery$ steps between branching performs $\lceil T / \BranchEvery\rceil$ segments and scores $K$ candidates in each.
For $\BranchEvery{=}64$ and $T{=}128$, the $2$ segments give $2 \times K \times \BranchEvery = 128K$ denoising passes plus $2K$ PRM evaluations, one per segment including the final, for $130K$ in total.
PRM Guided\,$K{=}8$ spends $1{,}040$ passes against $1{,}032$ for ORM Rerank$@8$, and PRM Guided\,$K{=}32$ spends $4{,}160$ against $4{,}128$ for ORM Rerank$@32$, both $+0.78\%$. The budgets are matched within $1\%$, and the small asymmetry favors PRM Guided.

\paragraph{Wall-clock validation of forward-pass matching.}
Hardware: $1 \times$ NVIDIA H20 $96$GB GPU, batch size 1 per problem, bfloat16 inference. The table reports per-problem wall-clock time averaged over the $1{,}319$ test problems in two independent runs, $2{,}638$ measurements in total.

\begin{center}
\small
\begin{tabular}{lrr}
\toprule

\rowcolor{ArxivTableHead}
\tablehead{Method} & \tablehead{Predicted passes} & \tablehead{Measured wall-clock (s/problem)} \\
\midrule

\mdseries
Vanilla $K{=}1$ (Dream-7B) & $128$ & $20.71$ \\
PRM Guided $K{=}8, \BranchEvery{=}64$ & $1{,}040$ & $155.95$ \\
PRM Hybrid $K{=}8, \BranchEvery{=}64$ & $1{,}040$ & $155.97$ \\
PRM Guided $K{=}32, \BranchEvery{=}64$ & $4{,}160$ & $612.61$ \\
Vanilla $K{=}1$ (LLaDA-8B) & $128$ & $23.21$ \\
PRM Guided $K{=}8$ (LLaDA) & $1{,}040$ & $172.02$ \\
ORM scoring, per candidate & $1$ & $\approx 0.3$ \\
\bottomrule
\end{tabular}
\end{center}

The wall-clock data supports forward-pass matching as a compute proxy. Measured ratios stay within $9\%$ of the predicted ones: $7.53\times$ against $8.13\times$ for $K{=}8$ over $K{=}1$, and $29.6\times$ against $32.5\times$ for $K{=}32$ over $K{=}1$, with the remaining discrepancy explained by per-segment overhead amortization. LLaDA is about $11\%$ slower at matched passes because of its larger per-pass FLOPs, and its $K{=}8$ over $K{=}1$ ratio stays close to Dream's. ORM scoring costs about $0.3$\,s per candidate against $20.71$\,s of denoising. In wall-clock terms, ORM Rerank@$8$ takes about $168$\,s per problem, slightly more than the $156$\,s of PRM Guided at $K{=}8$, but ORM Rerank@$6$ already reaches $72.40\%$ in about $126$\,s against $65.18\%$ (App.~\ref{app:full_pareto}), so the headline ordering also holds under wall-clock matching.

\section{Additional Matched Controls}
\label{app:additional_matched_controls}

This section reports the controls summarized in the main text. Entries with $\pm$ give sample standard deviations over independent runs; the 95\% CIs of all comparisons below are in App.~\ref{app:ci}.

\subsection{Matched-compute SMC and step-resolved pruning}
\label{app:smc}

The ESS-tempered SMC protocol uses $K{=}8$, $T{=}128$, scoring checkpoints after $64$ and $128$ denoising steps, and $1{,}040$ forward passes. The protocol, tempering grid, resampling rule, and terminal readouts were fixed before test evaluation.

\begin{center}
\small
\begin{tabular}{lccc}
\toprule

\rowcolor{ArxivTableHead}
\tablehead{Quantity} & \tablehead{SMC} & \tablehead{Top-1 guidance pool} & \tablehead{Independent samples} \\
\midrule

\mdseries
Oracle@$8$ (\%) & $77.89\pm0.57$ & $67.30\pm1.24$ & $81.05$ \\
Unique answers per problem & $3.95\pm0.02$ & $1.75$ & $4.31$ \\
Accuracy, weighted answer vote (\%) & $65.48\pm0.12$ & -- & -- \\
Accuracy, top-scoring candidate (\%) & $66.34\pm0.82$ & $65.18\pm0.75$ & -- \\
\bottomrule
\end{tabular}
\end{center}

SMC recovers $10.59$ of the $13.75$ Oracle points lost to top-$1$ pruning and raises unique answers from $1.75$ to $3.95$, yet both of its PRM readouts stay at the top-$1$ level and trail the $75.13\%$ of ORM Rerank@$8$ by $9.65$ and $8.79$\,pp. Top-$2$ retention, at $69.70\%$, trails it by $5.43$\,pp. All three gaps exclude zero (App.~\ref{app:ci}), so the terminal selector is the remaining bottleneck.

The step-resolved analysis is an offline counterfactual over the stored trajectories: at the initial, middle, and final stored states it applies a top-$M$ cut by PRM score and records how often the cut removes every lineage that eventually reaches a correct answer.

\begin{center}
\small
\begin{tabular}{lc}
\toprule

\rowcolor{ArxivTableHead}
\tablehead{Stored state and cut} & \tablehead{Risk of removing every correct lineage (\%)} \\
\midrule

\mdseries
Top-$1$, initial & 46.06 \\
Top-$1$, middle & 37.02 \\
Top-$1$, final & 19.92 \\
Top-$2$, middle & 22.54 \\
Top-$2$, final & 13.01 \\
Top-$4$, middle & 9.94 \\
Top-$4$, final & 5.93 \\
\bottomrule
\end{tabular}
\end{center}

From the middle to the final state, the top-$1$ risk falls by $17.1$\,pp, a drop whose interval excludes zero. The largest removal risk sits at early states, where PRM ROC-AUC is lowest (\Cref{sec:mechanism-decay}), and wider cuts shrink it at both the middle and final states.

\subsection{Final-state specialist control and task-specific MATH and MBPP}
\label{app:final_state}

The final-state PRM is the bidirectional PRM retrained under the ORM protocol of App.~\ref{app:orm_protocol_controls}, on final states from the official GSM8K training split and without a step embedding. It reranks the same candidate pool used for the ORM comparison:

\begin{center}
\small
\begin{tabular}{lcc>{\columncolor{ArxivTableCol}}c}
\toprule

\rowcolor{ArxivTableHead}
\tablehead{Candidate budget} & \tablehead{Cross-mask PRM} & \tablehead{Final-state PRM} & \tablehead{ORM Rerank} \\
\midrule

\mdseries
$N{=}8$ & 42.84 & $75.40\pm0.05$ & 75.13 \\
$N{=}32$ & 65.35 & $82.79\pm0.11$ & 82.71 \\
\bottomrule
\end{tabular}
\end{center}

At both budgets the paired interval between ORM Rerank and the final-state PRM includes zero (App.~\ref{app:ci}): final-state training closes the terminal scoring gap that cross-mask training opens.

For task-specific MATH, the fitting split contains $450$ problems and the validation split $50$; the test set contains all $500$ held-out MATH500 problems. At $N{=}K{=}8$, ORM Rerank reaches $30.65\%$ against $20.80\%$ for both PRM Rerank and deterministic PRM Guided, a $9.85$\,pp lead. On shared stored candidates, ORM Rerank reaches $31.90\%$ against $20.80\%$ for PRM Rerank, an $11.10$\,pp lead. Both intervals exclude zero.

On $257$ held-out MBPP tasks, ORM Rerank reaches $63.04\pm0.84\%$, PRM Rerank $65.47\pm0.74\%$, and PRM Guided $50.88\pm3.57\%$. ORM Rerank leads PRM Guided by $12.16$\,pp, an interval that excludes zero, while its difference from PRM Rerank includes zero: the PRM ranks final programs on par with the ORM, so the guided shortfall on MBPP reflects guidance rather than terminal scoring.

\subsection{Rollout-value supervision and frozen rollout diagnostic}
\label{app:rollout}

We relabeled $10{,}000$ states from official GSM8K training problems, using $8{,}000$ for fitting and a problem-disjoint $2{,}000$ for validation, with eight independent rollouts per state; a fractional label is the share of those rollouts that reach a correct answer. Binary and fractional scorers share the states, validation split, initialization, data order, checkpoint rule, and 500-step optimization budget. Pooled ROC-AUC against rollout outcomes rises from $0.816$ to $0.827$, a paired gain of $0.011$ whose interval excludes zero.

On all $1{,}319$ GSM8K test problems under the matched supervision control, inference accuracy is $61.49\%$ for fractional labels and $60.88\%$ for binary labels, a $0.61$\,pp difference whose interval includes zero.

The frozen diagnostic scores a fixed PRM against fresh rollout outcomes. Its ROC-AUC slope across mask buckets is $-0.326$, and switching the evaluation target from inherited labels to fresh rollouts shifts the slope by only $0.009$, an interval that includes zero: the decay survives the switch even though inherited-label disagreement rises from $0.50\%$ in the lowest mask bucket to $21.00\%$ in the highest. Absolute ROC-AUC values in this subset differ from the main-text $0.77$ to $0.54$ curve because it caps each mask bucket at $200$ states and scores against fresh rollout outcomes.

\section{Mechanism and Scorer Diagnostics}
\label{app:mechanism_diagnostics}

\subsection{Branching-frequency ablation}
\label{app:branching}

Sparse branching outperforms frequent branching on PRM Guided. The table reports one evaluation per interval at $K{=}8$; the run-averaged accuracy at $\BranchEvery{=}64$ is the $65.18\%$ of Table~\ref{tab:pareto}.

\begin{center}
\begin{tabular}{lcccc}
\toprule

\rowcolor{ArxivTableHead}
\tablehead{Step interval} & \tableheadmath{16} & \tableheadmath{32} & \tableheadmath{48} & \tableheadmath{64} \\
\midrule

\mdseries
Accuracy (\%) & $59.4$ & $62.8$ & $65.1$ & \best{66.5} \\
Forward passes per sample & $136K$ & $132K$ & $131K$ & $130K$ \\
\bottomrule
\end{tabular}
\end{center}

Accuracy rises at every step from $\BranchEvery{=}16$, the most frequent guidance, to $\BranchEvery{=}64$, the least frequent, even though frequent branching also spends more PRM calls. Frequent branching invokes the PRM at high mask ratios, where its ROC-AUC is lowest (Figure~\ref{fig:mechanism}(a)), and prunes branches on that weak signal (\Cref{sec:mechanism-diversity}).

\subsection{Confidence intervals for all reported comparisons}
\label{app:ci}

Every 95\% CI in the paper is collected here, so the other sections report point estimates only. The GSM8K headline intervals use a paired bootstrap with $2{,}000$ resamples of the $1{,}319$ test problems; the controls of App.~\ref{app:additional_matched_controls} resample their held-out problems or tasks, with the step-resolved row clustered by problem; the reversed-input row resamples its snapshots $1{,}000$ times. Full code is in the supplementary material. Accuracies are in \%, differences in pp.

\begin{center}
\small
\begin{tabular}{lcccc}
\toprule

\rowcolor{ArxivTableHead}
\tablehead{GSM8K headline budgets} & \tableheadmath{N{=}8} & \tablehead{95\% CI} & \tableheadmath{N{=}32} & \tablehead{95\% CI} \\
\midrule

\mdseries
Majority & $60.05$ & $[57.47, 62.62]$ & $67.63$ & $[65.20, 70.20]$ \\
\rowcolor{ArxivTableRow}
ORM Rerank & $75.13$ & $[72.86, 77.41]$ & $82.71$ & $[80.82, 84.61]$ \\
Oracle & $81.05$ & $[79.08, 83.24]$ & $91.13$ & $[89.61, 92.72]$ \\
Cross-mask PRM Rerank & $42.84$ & $[40.3, 45.6]$ & $65.35$ & $[62.8, 68.0]$ \\
\midrule

\mdseries
ORM Rerank $-$ Majority & $+15.09$ & $[+13.12, +17.44]$ & $+15.09$ & $[+12.96, +17.44]$ \\
Oracle $-$ Majority & $+21.00$ & $[+18.95, +23.28]$ & $+23.50$ & $[+21.15, +25.93]$ \\
Oracle $-$ ORM Rerank & $+5.91$ & $[+4.62, +7.13]$ & $+8.42$ & $[+7.05, +10.01]$ \\
\rowcolor{ArxivTableRow}
ORM Rerank $-$ PRM Guided, one run & $+8.79$ & $[+6.44, +11.14]$ & $+11.98$ & $[+9.78, +14.25]$ \\
ORM Rerank $-$ final-state PRM & -- & $[-0.38, +0.12]$ & -- & $[-0.31, +0.16]$ \\
\bottomrule
\end{tabular}
\end{center}

\begin{center}
\small
\begin{tabular}{lccl}
\toprule

\rowcolor{ArxivTableHead}
\tablehead{Other comparisons} & \tablehead{Estimate} & \tablehead{95\% CI} & \tablehead{Details} \\
\midrule

\mdseries
Vanilla GSM8K accuracy, $N{=}1$ & $43.14$ & $[40.6, 45.9]$ & App.~\ref{app:full_pareto} \\
ORM Rerank@$8$ $-$ SMC, weighted vote & $+9.65$ & $[7.71, 11.68]$ & App.~\ref{app:smc} \\
ORM Rerank@$8$ $-$ SMC, top particle & $+8.79$ & $[6.87, 10.74]$ & App.~\ref{app:smc} \\
ORM Rerank@$8$ $-$ top-$2$ retention & $+5.43$ & $[3.98, 6.87]$ & App.~\ref{app:topm} \\
Top-$1$ removal risk, middle $-$ final state & $+17.1$ & $[15.50, 18.72]$ & App.~\ref{app:smc} \\
MATH, ORM Rerank $-$ PRM Guided & $+9.85$ & $[7.20, 12.45]$ & App.~\ref{app:final_state} \\
MATH shared candidates, ORM Rerank $-$ PRM Rerank & $+11.10$ & $[8.75, 13.50]$ & App.~\ref{app:final_state} \\
MBPP, ORM Rerank $-$ PRM Guided & $+12.16$ & $[8.56, 15.86]$ & App.~\ref{app:final_state} \\
MBPP, ORM Rerank $-$ PRM Rerank & $-2.43$ & $[-5.54, 0.68]$ & App.~\ref{app:final_state} \\
Fractional $-$ binary labels, pooled ROC-AUC & $+0.0109$ & $[0.0031, 0.0185]$ & App.~\ref{app:rollout} \\
Fractional $-$ binary labels, GSM8K accuracy & $+0.61$ & $[-0.30, 1.52]$ & App.~\ref{app:rollout} \\
Frozen diagnostic, ROC-AUC slope across mask buckets & $-0.3260$ & $[-0.3843, -0.2670]$ & App.~\ref{app:rollout} \\
Slope shift, fresh rollouts $-$ inherited labels & $+0.0087$ & $[-0.0585, 0.0782]$ & App.~\ref{app:rollout} \\
Bidirectional PRM ROC-AUC, mask bucket $[0.0, 0.1)$ & $0.770$ & $[0.767, 0.774]$ & \Cref{sec:mechanism-decay} \\
Suffix ablation, change in ROC-AUC & $+0.001$ & $[-0.001, +0.004]$ & App.~\ref{app:m4m5_figure} \\
Reversed-input causal ROC-AUC & $0.500$ & $[0.498, 0.502]$ & App.~\ref{app:m4m5_figure} \\
LLaDA, bidirectional $-$ causal PRM guidance & $+9.39$ & $[8.5, 10.3]$ & App.~\ref{app:llada_prelim} \\
\bottomrule
\end{tabular}
\end{center}

Intervals that include zero belong to the controls where the paper claims parity or no effect: ORM Rerank against the final-state PRM, ORM Rerank against PRM Rerank on MBPP, fractional against binary labels on accuracy, the slope shift, and the suffix ablation; the reversed-input interval contains chance. Every other difference excludes zero. The single-run gaps of $8.79$ and $11.98$\,pp pair ORM Rerank with one PRM Guided run; Table~\ref{tab:pareto} reports the run-averaged $9.95$ and $12.69$\,pp. The final-state PRM row reports intervals only.

\subsection{Last-token pooling ablation (causal PRM)}
\label{app:last_token}

\paragraph{Motivation.}
Our main-text causal PRM (\Cref{sec:mechanism-causal}) uses mask-aware mean pooling to produce a scalar reward.
Causal reward models~\citep{ouyang2022instructgpt,stiennon2020summarize} read out the \emph{final-token} hidden state, since only the last position in a causal sequence attends to every earlier token.
The ablation asks whether the within-protocol bidirectional-vs-causal gap in \Cref{sec:mechanism-causal} reflects causal attention itself or the mean-pooling readout.

\paragraph{Setup.}
We retrain the causal PRM under an identical protocol to \Cref{sec:setup}, changing only the pooling strategy from mask-aware mean over the solution region to the hidden state of the \emph{last non-\texttt{MASK} non-\texttt{EOS} position}.
All other hyperparameters match App.~\ref{app:prm_training}.

\paragraph{Result.}
We score all $\NTestTrajectories$ held-out final-state trajectories with the retrained model.

\begin{center}
\small
\begin{tabular}{lc>{\columncolor{ArxivTableCol}}c}
\toprule

\rowcolor{ArxivTableHead}
\tablehead{Scorer (all at mask=0)} & \tablehead{ROC-AUC} & \tablehead{Delta from mean-pool causal} \\
\midrule

\mdseries
Causal PRM, mean pooling (original) & $0.6147$ & $0$ \\
Causal PRM, last-token pooling & \best{0.7274} & $+0.1127$ \\
Bidirectional PRM, mean pooling & $0.7759$ & $+0.1612$ \\
\midrule

\mdseries
Bidirectional PRM $-$ causal last-token PRM & \best{0.0485} & -- \\
\bottomrule
\end{tabular}
\end{center}

\paragraph{Interpretation.}
Last-token pooling improves the causal PRM at mask${=}0$ by $0.113$ ROC-AUC, closing about $70\%$ of the original $0.16$ mean-pool gap to the bidirectional PRM.
Most of the original gap came from choosing mean pooling over the last-token readout standard in causal reward models; a residual gap of about $0.05$ ROC-AUC persists with matched readouts.

\paragraph{Downstream reranking.}
\label{app:lasttoken_pareto}
Using the last-token causal PRM to rerank all $\NTestTrajectories$ trajectories removes the non-monotonic pathology of the mean-pool variant:

\begin{center}
\small
\begin{tabular}{lcccccc}
\toprule

\rowcolor{ArxivTableHead}
\tableheadmath{N} & \tableheadmath{1} & \tableheadmath{2} & \tableheadmath{4} & \tableheadmath{8} & \tableheadmath{16} & \tableheadmath{32} \\
\midrule

\mdseries
Causal PRM Rerank, mean-pool & $43.14$ & $42.68$ & $43.21$ & $43.90$ & $42.15$ & $40.56$ \\
Causal PRM Rerank, last-token & $43.14$ & $45.34$ & $48.29$ & $49.66$ & $49.36$ & \best{50.64} \\
$\Delta$ & $0.00$ & $+2.66$ & $+5.08$ & $+5.76$ & $+7.21$ & $+10.08$ \\
\bottomrule
\end{tabular}
\end{center}

The last-token variant rises with $N$ apart from a $0.30$\,pp dip at $N{=}16$ and reaches $50.64\%$ at $N{=}32$, $7.50$\,pp above $N{=}1$. The mean-pool variant declines at every step from $N{=}8$ onward, falls below the single-sample $43.14\%$ at $N{=}2$, $16$, and $32$, and ends at $40.56\%$, $2.58$\,pp below $N{=}1$ and below random selection at the same budget (App.~\ref{app:sanity}), so pooled ROC-AUC of $0.61$ does not carry over to within-problem ranking (Proposition~\ref{prop:pooled_main}). Once the readout is fixed and final-state ROC-AUC rises from $0.61$ to $0.73$, reranking improves with $N$ again.

\paragraph{Downstream reranking interpretation.}
The rerank collapse of the mean-pooled causal PRM is a readout artifact rather than a property of causal attention: with a last-token readout, the same backbone closes most of the final-state ROC-AUC gap and recovers the upward trend with $N$. Readout does not remove scorer-level differences. At $N{=}32$, last-token causal reranking reaches $50.64\%$ against $82.71\%$ for ORM Rerank, a $32.07$\,pp gap, and the bidirectional cross-mask PRM trails ORM Rerank by $32.3$\,pp at $N{=}8$ (\Cref{sec:empirical}).

\paragraph{Training-length controls.}
The main comparison uses $2{,}000$ training steps at batch $32$. Longer retrains of $15{,}000$ and $31{,}000$ steps (App.~\ref{app:one_epoch}) separate architectural effects from early-training convergence speed: the bidirectional-over-causal accuracy gap in the lowest mask bucket narrows from $13.6$\,pp at $15$K steps to $9.5$\,pp at $31$K, and a $7.8$\,pp gap remains under ORM-protocol training (App.~\ref{app:orm_protocol_controls}). Undertraining alone does not explain the architectural gap.

\subsection{PRM-as-reranker validation}
\label{app:sanity}

We validate the PRM-as-reranker result in \Cref{sec:empirical} with four checks on the same $\NTestTrajectories$ scored trajectories.

\paragraph{Sign-convention check.}
At mask${=}0$, the bidirectional PRM gives the $18{,}277$ correct trajectories a mean score of $-0.921$ and the $23{,}931$ incorrect ones $-3.675$, a difference of $2.75$: higher scores mark correct trajectories, so the sign convention is correct.

\paragraph{Rank correlation with correctness.}
Across the $1{,}182$ problems whose candidate pool contains both correct and incorrect trajectories, Kendall $\tau$ between PRM score and correctness has mean $0.397$ and median $0.441$, and the point-biserial correlation has mean $0.299$. The per-problem score separation, the mean score on correct candidates minus that on incorrect ones, is positive for $88\%$ of these problems, with mean $1.84$, median $1.85$, and a 10th to 90th percentile range from $-0.21$ to $3.78$. The ranking signal is consistently positive but moderate in magnitude.

\paragraph{Random Rerank comparator at matched compute.}
We compare PRM Rerank with Random Rerank, which picks uniformly among the $N$ candidates; Random Rerank entries give the mean $\pm$ std over $10$ trials.

\begin{center}
\small
\begin{tabular}{lcccc}
\toprule

\tableheadmath{N} & \tableheadmath{4} & \tableheadmath{8} & \tableheadmath{16} & \tableheadmath{32} \\
\midrule

\rowcolor{ArxivTableRow}
\mdseries
Random Rerank & $43.65 \pm 0.97$ & $43.24 \pm 1.26$ & $43.87 \pm 0.75$ & $43.28 \pm 1.16$ \\
Cross-mask PRM Rerank & $44.88$ & $42.84$ & $52.46$ & $65.35$ \\
PRM Rerank $-$ Random Rerank & $+1.23$ & $-0.41$ & $+8.60$ & $+22.07$ \\
\bottomrule
\end{tabular}
\end{center}

At $N{=}8$, PRM Rerank behaves like near-uniform selection: its $42.84\%$, slightly below the single-sample $43.14\%$, lies well within one standard deviation of Random Rerank, so the shortfall reflects sampling variation rather than a systematic preference for wrong candidates.
At $N{=}16$ and $N{=}32$, PRM Rerank exceeds Random Rerank by $8.6$ and $22.1$\,pp, so the scorer carries real information.
\label{app:scorer_breakdown}It still trails ORM Rerank by $32.3$\,pp at $N{=}8$ and $17.4$\,pp at $N{=}32$. This is the specialization cost of cross-mask training, which spreads the scorer across all mask ratios: the same architecture retrained on final states matches the ORM (App.~\ref{app:final_state}).

\paragraph{ROC-AUC by slice: reconciliation.}
The main text cites ROC-AUC values from several data slices; we list each slice explicitly:

\begin{center}
\small
\resizebox{\linewidth}{!}{%
\begin{tabular}{lcc}
\toprule

\rowcolor{ArxivTableHead}
\tablehead{Scorer} & \tablehead{Data slice} & \tablehead{ROC-AUC} \\
\midrule

\mdseries
Bidirectional PRM (mean-pool) & snapshot mask bucket $[0.0, 0.1)$, $n{=}126{,}624$ & $0.7702$ \\
Bidirectional PRM (mean-pool) & mask${=}0$ final states, $n{=}42{,}208$ & $0.7759$ \\
Causal PRM (mean-pool) & mask${=}0$ final states, $n{=}42{,}208$ & $0.6147$ \\
Causal PRM (last-token) & mask${=}0$ final states, $n{=}42{,}208$ & $0.7274$ \\
Bidirectional ORM (final-state-only training, \emph{reference only}) & mask${=}0$ final states, $n{=}42{,}208$ & $0.9623$ \\
\bottomrule
\end{tabular}}
\end{center}

The bidirectional ORM's $0.9623$ shows how well a \emph{scorer specialized for final states} performs on the same data; the architectural comparison against causal PRMs in \Cref{sec:mechanism-causal} uses matched PRM training. With the same training and reward head, the bidirectional PRM reaches $0.7759$ at mask${=}0$, against $0.6147$ for the mean-pooled causal PRM, which differs only in the attention mask, and $0.7274$ for the last-token causal PRM.

\subsection{Controlled ORM-protocol retrains}
\label{app:orm_protocol_controls}

The main comparison in Section~\ref{sec:mechanism-causal} holds the training protocol fixed, with training across mask ratios, a step embedding, and mean pooling, and varies only the attention mask. The ORM-protocol control tests whether the bidirectional advantage at mask${=}0$, $0.16$ ROC-AUC within protocol, persists under the ORM protocol, which trains on final states only and drops the step embedding.

The control uses two independent fits for each of three configurations: causal with last-token pooling, bidirectional with last-token pooling, and bidirectional with mean pooling. All are trained for $8{,}407$ steps on final states from the training problems. Together they separate the attention mask, the readout, and the training distribution.

\paragraph{Final-state classification accuracy at mask${=}0$.}
For these retrains we report held-out \emph{classification accuracy} at threshold $0.5$ rather than ROC-AUC. Accuracy and ROC-AUC are directionally consistent on this data.

\begin{center}
\small
\begin{tabular}{lcc}
\toprule
\rowcolor{ArxivTableHead}
\tablehead{Configuration (ORM protocol)} & \tablehead{Accuracy (\%)} & \tablehead{Sample std (pp)} \\
\midrule

\mdseries
Causal, last-token & $84.05$ & $0.15$ \\
Bidirectional, last-token & \best{91.83} & $0.08$ \\
Bidirectional, mean-pool & $91.78$ & $0.23$ \\
\bottomrule
\end{tabular}
\end{center}

\paragraph{Residual bidirectional-over-causal gap.}
With matched last-token readouts, the bidirectional scorer leads the causal one by $91.83 - 84.05 = \mathbf{7.78}$\,pp, with paired standard error below $0.2$\,pp. The gap has the same sign as under the PRM protocol and is close to its $9.5$\,pp value at $31$K steps (App.~\ref{app:one_epoch}), so cross-mask training does not account for the whole causal deficit. For the bidirectional scorer, readout barely matters: last-token and mean pooling reach $91.83\%$ and $91.78\%$.

\subsection{Longer-training convergence analysis}
\label{app:one_epoch}

The main PRMs train for $2{,}000$ steps. To test whether the bidirectional-over-causal gap is an undertraining artifact, we retrained the mean-pooled bidirectional PRM and the last-token causal PRM under an identical protocol for $15{,}000$ steps and then for $31{,}000$ steps, about one full pass over the training corpus.

\paragraph{Per-bucket accuracy at 31K steps.}

\begin{center}
\small
\begin{tabular}{lcc>{\columncolor{ArxivTableCol}}c}
\toprule

\rowcolor{ArxivTableHead}
\tablehead{Mask ratio bucket} & \tablehead{Bidirectional, mean-pool (\%)} & \tablehead{Causal, last-token (\%)} & \tablehead{Gap (pp)} \\
\midrule

\mdseries
$[0.0, 0.1)$ & $90.12$ & $80.66$ & $+9.46$ \\
$[0.1, 0.2)$ & $87.87$ & $79.08$ & $+8.79$ \\
$[0.2, 0.3)$ & $82.12$ & $71.51$ & $+10.61$ \\
$[0.3, 0.4)$ & $75.96$ & $66.12$ & $+9.84$ \\
$[0.4, 0.5)$ & $78.05$ & $70.12$ & $+7.93$ \\
$[0.5, 0.6)$ & $81.82$ & $68.18$ & $+13.64$ \\
$[0.6, 0.7)$ & $72.08$ & $59.74$ & $+12.34$ \\
$[0.7, 0.8)$ & $70.30$ & $69.09$ & $+1.21$ \\
$[0.8, 0.9)$ & $68.82$ & $59.41$ & $+9.41$ \\
$[0.9, 1.0]$ & $57.80$ & $56.27$ & $+1.53$ \\
\bottomrule
\end{tabular}
\end{center}

\paragraph{Gap evolution from 15K to 31K steps.}
In the lowest mask bucket $[0.0, 0.1)$, the bidirectional-over-causal gap evolves as follows:

\begin{center}
\small
\begin{tabular}{lcc>{\columncolor{ArxivTableCol}}c}
\toprule

\tablehead{Training steps} & \tablehead{Bidirectional (\%)} & \tablehead{Causal, last-token (\%)} & \tablehead{Gap (pp)} \\
\midrule

\mdseries
$15$K & $91.77$ & $78.19$ & $13.58$ \\
$31$K & $90.12$ & $80.66$ & \best{9.46} \\
\midrule

\mdseries
Change, $31$K $-$ $15$K & $-1.65$ & $+2.47$ & \best{\textminus4.12, \textminus30\%} \\
\bottomrule
\end{tabular}
\end{center}

Training loss decreases for both architectures through $31$K steps. In this bucket the bidirectional accuracy stays nearly flat, a $1.65$\,pp drop consistent with a plateau, while the causal accuracy improves by $2.47$\,pp, so causal training is still converging at $15$K.

\paragraph{Residual architectural effect persists.}
At $31$K steps the gap in the lowest mask bucket narrows by $30\%$, from $13.6$ to $9.5$\,pp, so part of the $15$K-step gap is undertraining. The remaining $9.5$\,pp persists after doubling the training budget, the bidirectional PRM stays ahead in all ten mask buckets, and the ORM-protocol control shows a similar $7.8$\,pp gap under a different training recipe (App.~\ref{app:orm_protocol_controls}). Together they point to a real architectural effect of bidirectional attention that longer training narrows but does not close.

\section{Additional Experimental Results}
\label{app:additional_results}

\subsection{Full compute-matched Pareto table}
\label{app:full_pareto}

Table~\ref{tab:pareto} in the main text reports the headline budgets $N{\in}\{8, 32\}$. The full sweep below covers every tested budget on the $1{,}319$ GSM8K test problems with Dream-v0-Instruct-7B; PRM Guided uses $K{=}N$ and $\BranchEvery{=}64$ and reports mean $\pm$ sample std over three runs for $K\le 24$ and two for $K{=}32$.

\begin{center}
\small
\begin{tabular}{r>{\columncolor{ArxivTableCol}}c cccc}
\toprule

\rowcolor{ArxivTableHead}
\tableheadmath{N} & \tablehead{Majority} & \tablehead{ORM Rerank} & \tablehead{Weighted-Majority} & \tablehead{Oracle} & \tablehead{PRM Guided} \\
\midrule

\mdseries
$1$  & $43.14$ & $43.14$ & $43.14$ & $43.14$ & $41.60 \pm 1.31$ \\
$2$  & $43.14$ & $55.88$ & $55.88$ & $57.77$ & $51.40 \pm 1.33$ \\
$4$  & $52.99$ & $66.79$ & $66.94$ & $70.96$ & $59.21 \pm 0.93$ \\
$6$  & $57.24$ & $72.40$ & $72.71$ & $77.41$ & $66.21 \pm 0.76$ \\
$8$  & $60.05$ & \best{75.13} & $75.36$ & $81.05$ & $65.18 \pm 0.75$ \\
$12$ & $63.46$ & $77.79$ & $77.94$ & $84.31$ & $69.42 \pm 0.64$ \\
$16$ & $65.58$ & $79.68$ & $79.68$ & $86.28$ & $67.32 \pm 0.59$ \\
$24$ & $67.25$ & $80.89$ & $80.97$ & $89.01$ & $73.67 \pm 0.82$ \\
$32$ & $67.63$ & \best{82.71} & $82.71$ & $91.13$ & $70.02 \pm 0.70$ \\
\bottomrule
\end{tabular}
\end{center}

The sweep supports three conclusions. First, ORM Rerank grows monotonically from $43.14\%$ at $N{=}1$ to $82.71\%$ at $N{=}32$, while Oracle at $91.13\%$ leaves $8.4$\,pp of headroom that neither ORM Rerank nor PRM Guided exploits. Second, verifier-weighted voting, the Weighted-Majority column~\citep{li2023makinglargelanguagemodels}, stays within $0.4$\,pp of ORM Rerank at every budget, so the gain does not hinge on the argmax selection rule. Third, the gap between ORM Rerank and PRM Guided grows from $9.95$\,pp at $N{=}8$ to $12.69$\,pp at $N{=}32$. Majority, ORM Rerank, Weighted-Majority, and Oracle are computed once from the pooled $32$-sample candidate set per problem, so they carry no across-run standard deviation; their CIs are in App.~\ref{app:ci}.

\paragraph{Note on PRM Guided at $K{=}1$.}
At $K{=}1$, PRM Guided reaches $41.60 \pm 1.31\%$, about one standard deviation below the $43.14\%$ of Vanilla. The two are algorithmically equivalent, since $K{=}1$ removes any selection, but segmental decoding re-enters per-token sampling at each segment boundary and so perturbs the random stream relative to single-pass Vanilla decoding. The $K{=}1$ entry is reported for completeness and is not part of the headline comparison at $K{=}8$ and $K{=}32$.

\subsection{PRM Hybrid Oracle ceiling}
\label{app:hybrid_orm_ceiling}

The headline ORM Rerank vs PRM Guided comparison mixes two effects: the PRM signal may be weak, and top-$1$ pruning may use that signal inefficiently. PRM Hybrid separates them. It is identical to PRM Guided except that all $K{=}8$ candidates survive the final segment, and the Oracle ceiling of that pool measures what guidance leaves for any selector. PRM Hybrid spends $1{,}040$ passes, within $0.8\%$ of ORM Rerank@$8$, and its entry is the mean $\pm$ std over three runs.

\begin{center}
\small
\begin{tabular}{lc}
\toprule

\rowcolor{ArxivTableHead}
\tablehead{Candidate pool and selector} & \tablehead{GSM8K accuracy (\%)} \\
\midrule

\mdseries
PRM Hybrid pool, Oracle@$8$ & $67.30 \pm 1.24$ \\
\midrule

\mdseries
Independent samples, ORM Rerank@$8$ & $75.13$ \\
Independent samples, Oracle@$8$ & $81.05$ \\
\bottomrule
\end{tabular}
\end{center}

\paragraph{Perfect-selector ceiling.}
Even a perfect selector over the PRM Hybrid pool reaches only $67.30\%$, $7.83$\,pp below ORM Rerank@$8$ and $13.75$\,pp below Oracle@$8$ on independent samples (Section~\ref{sec:mechanism-diversity}), so part of the PRM Guided shortfall sits upstream in the candidate pool, before any selection rule applies.

\subsection{MATH500 out-of-distribution evaluation}
\label{app:math500_full}

MATH500~\citep{hendrycks2021math} provides an out-of-distribution evaluation with $500$ problems for the GSM8K-trained scorers. GSM8K uses short integer answers with a clear extraction rule, which makes final-answer verification comparatively direct; the task-specific MATH and MBPP controls test the same ordering under different answer formats. We use the LaTeX-robust \texttt{math-verify} grader because the GSM8K-style regex extractor is unreliable on MATH500's LaTeX answer format.

\paragraph{Sampler setting.}
The independent-sample methods on MATH500 use temperature $1.0$ and $\texttt{alg\_temp}{=}0$, with $32$ trajectories per problem in each of two runs; PRM Guided inherits the GSM8K sampler, with temperature $0.5$ and $\texttt{alg\_temp}{=}0.5$, and its run counts are given below the table. Because the samplers differ, this table tests whether the GSM8K-trained ORM transfers, not the PRM-vs-ORM ordering. The sampler-controlled comparison is the task-specific MATH control in App.~\ref{app:additional_matched_controls}, where a MATH-trained ORM beats PRM Guided by $9.85$\,pp on the same $500$ test problems.

The bidirectional Dream ORM trained on GSM8K is applied to MATH500 final states:

\begin{center}
\small
\begin{tabular}{lrrrrrr}
\toprule

\rowcolor{ArxivTableHead}
\tablehead{Method} & \tableheadmath{N{=}1} & \tableheadmath{N{=}2} & \tableheadmath{N{=}4} & \tableheadmath{N{=}8} & \tableheadmath{N{=}16} & \tableheadmath{N{=}32} \\
\midrule

\mdseries
Vanilla, first trajectory & $6.90$ & \multicolumn{5}{c}{same at every $N$} \\
Majority              & $6.90^\dagger$ & $7.10$ & $8.70$ & $13.20$ & $15.40$ & \best{17.20} \\
ORM Rerank            & $6.90^\dagger$ & $6.60$ & $7.40$ & $6.70$ & $7.50$ & $6.10$ \\
PRM Guided, $K{=}N$ & $11.25$ & $13.45$ & $14.00$ & $13.72$ & $14.35$ & $15.00$ \\
Oracle & $6.90^\dagger$ & $12.00$ & $21.90$ & $31.10$ & $42.70$ & $54.10$ \\
\bottomrule
\end{tabular}
\end{center}

\paragraph{Out-of-distribution behavior.}
The GSM8K-trained ORM does not transfer to MATH500: ORM Rerank stays near $7\%$ across all $N$, below both Majority@$32$ ($17.20\%$) and PRM Guided with $K{=}32$ ($15.00\%$), even though Oracle@$32$ reaches $54.10\%$. The ORM Rerank advantage requires a verifier trained for the task, and a MATH-trained ORM restores the ordering (App.~\ref{app:additional_matched_controls}).

\noindent{\footnotesize $^\dagger$~At $N{=}1$, Majority, ORM Rerank, and Oracle reduce by definition to the Vanilla single-trajectory accuracy ($6.90\%$). Entries are means over independent runs, two for Majority, ORM Rerank, and Oracle, and for PRM Guided four at $K\in\{1,2,4,16\}$, five at $K{=}8$, and two at $K{=}32$; PRM Guided uses $\BranchEvery{=}64$, and no sample standard deviation exceeds $1.84$\,pp. Majority@$N$ for $N\in\{2,4,8,16,32\}$ is computed on the $32$ Vanilla trajectories using \texttt{math\_verify} answer clustering, that is, majority voting on parsed answers.}\normalsize

\subsection{\texorpdfstring{Top-$M$ retention ablation: relaxing top-$1$ pruning}{Top-M retention ablation: relaxing top-1 pruning}}
\label{app:topm}

Deterministic top-$1$ pruning may be too aggressive. We test this with the top-$M$ retention generalization of Algorithm~1: at each segment, keep the top-$M$ candidates by PRM score instead of top-$1$; each retained state spawns $K/M$ children at the next segment, preserving total branch width $K$.

We evaluate $K{=}8$, $\BranchEvery{=}64$ with $M \in \{1, 2\}$ on the $1{,}319$ GSM8K test problems; $M{=}1$ is PRM Guided. Each evaluation uses about $1{,}040$ forward passes per sample, and entries are the mean $\pm$ std over three runs.

\begin{center}
\small
\begin{tabular}{lrc}
\toprule

\rowcolor{ArxivTableHead}
\tablehead{Method} & \tableheadmath{M} & \tablehead{GSM8K accuracy (\%)} \\
\midrule

\mdseries
PRM Guided, top-$1$ pruning & 1 & $65.18 \pm 0.75$ \\
PRM Guided, top-$M$ retention  & 2 & \best{69.70 \textpm{} 1.42} \\
ORM Rerank@$8$, independent samples & -- & $75.13$ \\
\bottomrule
\end{tabular}
\end{center}

Top-$M{=}2$ improves accuracy over $M{=}1$ by $4.52$\,pp, from $65.18\%$ to $69.70\%$, so relaxed pruning helps. It still trails ORM Rerank@$8$ by $5.43$\,pp at the same budget, matching the main result that independent sampling plus a final-state ORM is stronger on Dream-7B GSM8K.

Top-$M$ retention reduces the damage of top-$1$ pruning, and the matched SMC control in App.~\ref{app:additional_matched_controls} restores most of the pool while its PRM-selected accuracy stays at the top-$1$ level. Both point to pairing diversity-preserving search with a final-state verifier as the direct next step.

\subsection{Cross-backbone results on LLaDA-8B-Base (full evaluation)}
\label{app:llada_prelim}

Figure~\ref{fig:llada_prelim} in \Cref{sec:mechanism-causal} summarizes the full LLaDA evaluation grid on the complete GSM8K test set. For each scorer, the evaluated configurations across the four $\BranchEvery$ values, including $48$, are pooled into $8$ cells, with $K{=}8$ throughout and $1{,}319$ problems per cell; the bidirectional-over-causal conclusion is unchanged across pooling choices.

\paragraph{Summary statistics.}

\begin{center}
\small
\begin{tabular}{lccc>{\columncolor{ArxivTableCol}}c}
\toprule

\rowcolor{ArxivTableHead}
\tablehead{Method} & \tablehead{Accuracy (\%)} & \tablehead{Sample std} & \tablehead{Pooled cells} & \tablehead{Delta vs Vanilla (pp)} \\
\midrule

\mdseries
Vanilla, $K{=}1$ & $20.77$ & -- & 1 & $0$ \\
Causal PRM Guided, $K{=}8$  & $22.25$ & $0.90$ & 8 & $+1.48$ \\
Bidirectional PRM Guided, $K{=}8$   & \best{31.64} & $0.75$ & 8 & \best{+10.87} \\
\bottomrule
\end{tabular}
\end{center}

Across all tested LLaDA configurations, bidirectional PRMs consistently outperform causal PRMs, by $9.39$\,pp on average, a gap whose interval excludes zero (App.~\ref{app:ci}). The gap has the same sign as the bidirectional advantage on Dream-7B (\Cref{sec:mechanism-causal}), reproducing the effect on a second, independently released dLLM backbone.

\section{Residual Mechanisms and Scope}
\label{app:residual_mechanisms}

\subsection{Full limitations discussion}
\label{app:limitations_full}

The main body's \textbf{Threats to validity} paragraph (Section~\ref{sec:discussion}) summarizes the following list:

(1)~\emph{Domain scope.} Our evidence covers math reasoning (GSM8K, MATH500) and code generation (MBPP); whether the same failure mode holds for creative generation or open-ended text is untested. We conjecture the failure is less severe when partial states are less reasoning-critical.

(2)~\emph{Cross-backbone scope.} LLaDA-8B-Base reproduces the bidirectional advantage across all $16$ pooled evaluation cells on the full $1{,}319$-problem GSM8K test set (App.~\ref{app:llada_prelim}): bidirectional PRM guidance reaches $31.64\%$ against $22.25\%$ for causal PRM guidance and $20.77\%$ for Vanilla. An LLaDA-specific ORM comparison is not yet run.

(3)~\emph{Supervision.} Our PRMs are trained with binary final-correctness labels only; step-level rationale supervision~\citep{lightman2023letsverify} could alter the signal-decay curve, though it remains unclear whether it would fix causal attention's weakness on fully decoded final states.

(4)~\emph{Readout dependence.} The main causal-vs-bidirectional comparison (Section~\ref{sec:mechanism-causal}) holds mask-aware mean pooling fixed. The last-token ablation in the same section closes about $70\%$ of the final-state ROC-AUC gap; appended-\texttt{[CLS]} and learned query pooling remain untested.

(5)~\emph{Guidance-algorithm scope.} We evaluate one family of PRM Guided algorithms, segmental top-$1$ pruning with a fixed branching interval, together with matched ESS-tempered SMC and top-$M$ controls. Adaptive branching schedules aligned with denoising stages where the PRM is most competent, lookahead-based scoring, SMC variants with explicit diversity kernels, stochastic beam search, and hybrid PRM and ORM voting remain untested and are natural follow-ups to our diversity-collapse finding.

(6)~\emph{Sampler sensitivity.} Snapshot distributions depend on the dLLM's sampling hyperparameters (temperature, top-$p$, noise schedule). We use the Dream-7B default schedule; whether the ROC-AUC decay shape is sampler-invariant is an open empirical question.

(7)~\emph{ROC-AUC is not guidance utility.} Snapshot-level ROC-AUC measures discrimination but not the utility of the discrimination for search. A scorer with moderate ROC-AUC but well-calibrated uncertainty could in principle outperform a higher-AUC but overconfident scorer; our diversity-collapse finding (Section~\ref{sec:mechanism-diversity}) points to this gap.

(8)~\emph{Falsification cost.} Testing hypothesis (H1) of App.~\ref{app:hypothesis_family} requires pretraining a new $\sim\!7$B permutation-LM scorer from scratch and is beyond the scope of this diagnostic paper.

\subsection{Side-of-context rule-out controls}
\label{app:m4m5_figure}

Figure~\ref{fig:m4m5} shows the two controls summarized in Section~\ref{sec:mechanism-ruleout}.

\begin{figure}[ht]
    \centering
    \includegraphics[width=0.88\textwidth]{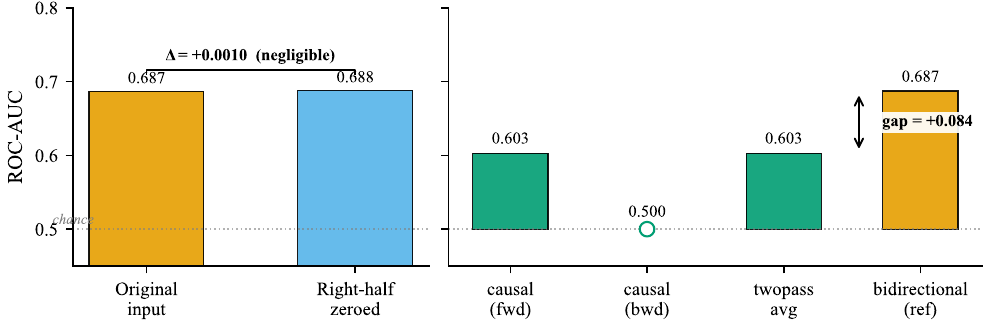}
    \caption{\textbf{Two rule-out controls on the bidirectional-over-causal advantage.} \textbf{Left: suffix ablation.} Zeroing right-half generated tokens changes the bidirectional PRM's ROC-AUC by about $+0.001$, so the advantage does not rest on right-half suffix tokens. \textbf{Right: reversed-input causal scoring.} The causal PRM scores reversed input at chance (ROC-AUC $\approx 0.500$), as RoPE pretraining predicts (App.~\ref{app:m5_caveats}), so averaging it with the forward pass recovers only the forward ROC-AUC and leaves a $0.084$ gap to the bidirectional reference. Both panels share the ROC-AUC axis, with chance at $0.5$.}
    \label{fig:m4m5}
\end{figure}

\textbf{Suffix ablation (left panel).} Zeroing the right-half generated tokens of each scored snapshot changes the bidirectional PRM's ROC-AUC by only about $+0.001$, from $0.687$ to $0.688$ over $\NSnapshots$ snapshots, so the bidirectional advantage does not come from reading right-half suffix tokens. The ablation targets that simplest form of right context; finer uses of side-of-context information remain possible.

\textbf{Reversed-input causal scoring (right panel).} Scoring with the causal PRM forward (L$\to$R) and on reversed input (R$\to$L), then averaging, does not recover the bidirectional ROC-AUC. The reversed pass scores ROC-AUC $0.500$, at chance (App.~\ref{app:ci}), so the two-pass average matches the forward pass alone. This chance-level result is what the RoPE-based Dream backbone, which never saw reversed text in training, predicts (App.~\ref{app:m5_caveats}), which makes the control a test of inference-time fixes on standard backbones rather than a behavioral rule-out of side-of-context information.

Together, the two controls rule out the simplest explanation and the simplest inference-time fix for the bidirectional-over-causal gap; App.~\ref{app:hypothesis_family} lists the hypotheses that remain.

\subsection{Reversed-input control and RoPE}
\label{app:m5_caveats}

The reversed-input experiment scores the causal PRM on reversed input sequences and yields ROC-AUC $0.500$. This outcome follows from the positional encoding, so the experiment constrains \emph{inference-time fixes} rather than ruling out side-of-context information behaviorally.

Dream-7B's backbone uses Rotary Position Embedding (RoPE, \citealp{su2024roformer}). RoPE injects position into self-attention through rotations of the query and key vectors: for position $p$ and frequency $\theta_k$, the vector at position $p$ is rotated by an angle $p \theta_k$, so each query-key inner product depends on the relative offset between the two positions. The backbone was trained on L$\to$R inputs, where these offsets encode \emph{ordered} neighbor relationships.

Reversing the input sequence places $t_{L-p-1}$ at position $p$ while RoPE keeps the $p$-indexed rotation, which flips the sign of every relative offset between tokens: a token's former right neighbors now sit to its left, and under the causal mask it attends only to them. The backbone never saw text in this order, and the causal PRM was trained only on forward inputs, so the scalar its reward head reads out carries no signal about correctness and ROC-AUC falls to the null $0.500$.

This control does not test whether a causal PRM benefits from right-to-left evidence in principle; it tests whether a causal PRM with pretraining-locked positional encoding can score reversed sequences zero-shot. In this setup, it cannot. A decisive rule-out of side-of-context information would require (i)~retraining the backbone on bidirectional input distributions, which amounts to bidirectional pretraining, (ii)~replacing RoPE with a symmetric or permutation-invariant encoding, or (iii)~the permutation-LM falsification test in App.~\ref{app:hypothesis_family}.

\subsection{Residual hypothesis family and falsification design}
\label{app:hypothesis_family}

Section~\ref{sec:mechanism-ruleout}'s rule-outs narrow but do not uniquely identify the bidirectional-over-causal mechanism. Three residual hypotheses remain:

\textbf{(H1) Joint unordered modeling capacity.} The scorer needs to be able to condition on an arbitrary observed-token subset $O \subseteq \{1,\dots,L\}$ without committing to a single sequential factorization of the unobserved positions. Bidirectional attention provides this by construction: each position's hidden state attends to every observed token regardless of ordering. Causal attention provides this only at the end-of-sequence token, and only if the sequence has been fully denoised; for partial states with scattered masking, causal attention commits to a single L$\to$R factorization that may not be optimal. Proposition~\ref{prop:causal_bidir_identifiability} formalizes this gap for prefix-causal readouts.

\textbf{(H2) Readout-pattern dependence.} The scorer needs a specific readout that can aggregate information after observing all tokens, such as last-token, appended-\texttt{[CLS]}, or learned query pooling. The last-token pooling retrain in Section~\ref{sec:mechanism-causal} closes about $70\%$ of the final-state ROC-AUC gap. A full ablation with an appended-\texttt{[CLS]} readout is a natural next step.

\textbf{(H3) Training-distribution shift on causal backbones.} Supervising a causal backbone on scattered-mask intermediate states may disrupt its left-to-right prefix inductive bias, producing attention patterns that cannot cleanly recover standard AR scoring behavior even on fully decoded (mask${=}0$) inputs.

Among these, (H1) is the most actionable. A \emph{permutation-language-model scorer}~\citep{yang2019xlnet} would have joint unordered modeling capacity by training (arbitrary orderings of unobserved positions are randomly sampled during pretraining) while remaining strictly autoregressive at inference time. If such a scorer closes the bidirectional-over-causal gap on dLLM intermediate states, (H1) is supported. If it does not, the bottleneck lies in (H2), (H3), or an explanation not listed here. We propose this as the next experimental step (Section~\ref{sec:discussion}).

\subsection{Design constraints for future PRM guidance}
\label{app:discussion_full}

Four design principles follow: (1)~\emph{architecture matters}, since causal scorers trail bidirectional ones on masked states even with a matched readout; (2)~\emph{compute matching matters}, because scorer calls change the effective inference budget; (3)~\emph{diversity matters}, since sparse branching recovers accuracy and wider retention keeps more correct lineages alive; and (4)~\emph{a final-state verifier is the natural terminal selector}, since SMC restores most of the pool without raising PRM-selected accuracy while a final-state PRM matches the ORM on the same candidates.

\section{Additional Formalism and Case Study}
\label{app:formalism_case}

\subsection{Theoretical bounds and proofs}
\label{app:theory}

\label{sec:theoretical-bounds}

The bounds below serve as diagnostics: they explain why AUC decay and weak reranking arise under natural null models, without assuming that the learned PRM is Bayes-optimal in any mask bucket.

\begin{proposition}[Information ceiling for masked-state AUC]
\label{prop:mi_auc}
Fix a mask bucket $t$, let $P_t$ be the on-policy law of $(\bx_t,y)$, with $y\in\{0,1\}$, class priors $\pi_y>0$, and $P_y=P_t(\bx_t\mid y)$. For any scorer $r:\mathcal X\to\mathbb R$,
$\operatorname{AUC}_t(r)\le \tfrac12+\mathrm{TV}(P_1,P_0) \le \tfrac12+\sqrt{I(\bx_t;y)/(2\pi_0\pi_1)}$
(capped at $1$). If $I(\bx_t;y)\to 0$ as $\rho(t)\to1$, then $\sup_r \operatorname{AUC}_t(r)\to\tfrac12$.
\end{proposition}
\begin{proof}[Sketch]
Let $Q_y$ be the law of the score $r(\bx_t)$ given $y$, and $F_0^{\mathrm{mid}}(s)=\Pr_{Q_0}(S<s)+\tfrac12\Pr_{Q_0}(S=s)$, a function with values in $[0,1]$. Then
$\operatorname{AUC}_t(r)-\tfrac12=\mathbb E_{Q_1}F_0^{\mathrm{mid}}-\mathbb E_{Q_0}F_0^{\mathrm{mid}}\le\mathrm{TV}(Q_1,Q_0)\le\mathrm{TV}(P_1,P_0)$, the last step by data processing; Pinsker gives $I(\bx_t;y)\ge 2\pi_0\pi_1\mathrm{TV}(P_1,P_0)^2$.
\end{proof}

This explains why AUC decay with mask ratio is expected under partial-state label ambiguity; it does not imply the empirical PRM attains the ceiling.

\begin{proposition}[Pooled AUC need not imply within-problem reranking]
\label{prop:pooled_auc_rerank}
Let $A(r)$ be pooled AUC and $T_N(r)$ the expected within-problem top-1 selection accuracy with $N$ candidates. For every $N\ge 2$ and $\epsilon\in(0,1)$, there exists a distribution and scorer with $A(r)=1-\epsilon$ whose $T_N(r)$ equals that of uniform random selection.
\end{proposition}
\begin{proof}[Sketch]
Let each problem have a latent type $z\in\{0,1\}$, uniform across problems, whose candidates are correct with probability $1-\epsilon$ if $z{=}1$ and $\epsilon$ if $z{=}0$, and set $r(x)=z$. Pooled over problems, $A=1-\epsilon$. Within a problem all candidates share $z$ and hence the same score, so uniform tie-breaking makes top-1 selection equivalent to random selection.
\end{proof}

In a problem with one correct and $N{-}1$ incorrect candidates, uniform random selection gives $T_N=1/N$, and a pairwise success rate $q$ under conditional independence gives $T_N=q^{N-1}$: $q=0.9$ yields $T_{32}\approx 0.038$. High pooled discrimination and poor within-problem ranking can coexist.

\begin{corollary}[Kendall-$\tau$ ceiling on top-1 rerank]
\label{cor:kendall_top1}
Consider problems with one correct and $N{-}1$ incorrect candidates. For a problem of type $z$, let $X^+$ and $X^-$ be its correct and an incorrect candidate, $q_z=\Pr(r(X^+)>r(X^-)\mid z)$, and $\tau_z=2q_z-1$. If the $N{-}1$ comparisons against incorrect candidates are conditionally independent and $\tau_z\le\tau_0$ for all problems, then $T_N(r)\le((1+\tau_0)/2)^{N-1}$.
\end{corollary}

\emph{Remark.} With a uniform $\tau_0{=}0.40$, close to the observed mean $\tau$ of $0.397$ (App.~\ref{app:sanity}), the corollary gives $T_8\le 0.7^7\approx 0.082$, below the $1/8{=}0.125$ of random choice: moderate pairwise signal need not survive a top-$1$ pick. A single learned scorer makes correlated comparison errors, which this toy model ignores, so the corollary illustrates the mechanism rather than bounding PRM Rerank accuracy. Empirically, PRM Rerank@$8$ reaches $42.84\%$ (\Cref{sec:empirical}), at the level of Random Rerank (App.~\ref{app:sanity}).

\begin{proposition}[Prefix-causal identifiability gap]
\label{prop:causal_bidir_identifiability}
Let $s=(M,x_M)$ be a partial state with observed answer-position subset $M\subseteq[L]$. A prefix-causal scorer at frontier $k$, which reads only positions $1,\dots,k$, is measurable only w.r.t.\ $\mathcal G_k=\sigma(M\cap[k],x_{M\cap[k]})$. If the Bayes posterior $\eta(s)=\Pr(y=1\mid s)$ is not $\mathcal G_k$-measurable, no prefix-causal scorer can match the bidirectional Bayes scorer under any strictly proper loss, however much data it sees. Diffusion masks range over $2^L$ subsets (or $\binom{L}{m}$ in a fixed-size bucket); strict prefixes span only $L+1$ mask shapes.
\end{proposition}
\begin{proof}[Sketch]
Strict propriety forces positive excess risk whenever $\eta\ne\mathbb E[y\mid\mathcal G_k]$. Two states agreeing on the prefix but differing in right-context observed tokens are collapsed by $\mathcal G_k$ yet separable by the full bidirectional filtration.
\end{proof}

\emph{Scope}: the statement applies to prefix-causal PRM readouts; a causal transformer with a sequence-level readout after observing the whole masked sequence is not covered. It formalizes the combinatorial mismatch between autoregressive prefix information and non-prefix diffusion masks.


\subsection{Case study: why PRM-guided selection fails (GSM8K problem 526)}
\label{sec:case_study}

To make the failure concrete, we trace one GSM8K test problem (id $526$) through the diagnostics.

\paragraph{Problem.}
\emph{``Ada's daily electric consumption is $12$ kWh. She adds a device consuming $2$ kWh/day. At $\$1.50$/kWh, what is the weekly bill difference?''}
Gold answer: $\$21$ (the new device adds $2{\times}7{\times}\$1.50{=}\$21$/week).

\paragraph{Three candidate trajectories and their PRM scores.}
The $32$ independent candidates for this problem include:

\begin{center}
\small
\begin{tabular}{clcc}
\toprule

\rowcolor{ArxivTableHead}
\tablehead{Candidate} & \tablehead{Answer sketch} & \tablehead{Causal PRM score} & \tablehead{Correct?} \\
\midrule

\mdseries
$t{=}13$ & ``$2 \times \$1.50 = \$30$'' (wrong arithmetic, no weekly factor) & $+1.001$ & $\times$ \\
$t{=}25$ & ``diff $= 12{-}2{=}10$ kWh, $10{\times}\$1.50{=}\$15$'' (reversed sign) & $-0.827$ & $\times$ \\
$t{=}3$  & ``$2 \times \$1.50 \times 7 = \$21$'' (correct chain) & $-0.908$ & $\checkmark$ \\
\bottomrule
\end{tabular}
\end{center}

\paragraph{What the three findings predict about this case.}
(i)~\emph{Mask-ratio decay} (\Cref{sec:mechanism-decay}): at early denoising the PRM sees nearly empty states, where its discrimination is near chance (ROC-AUC $0.54$ in the most-masked bucket), so it has little basis to separate the chain that becomes $t{=}13$ from the one that becomes $t{=}3$.
(ii)~\emph{Diversity collapse} (\Cref{sec:mechanism-diversity}): a PRM Guided run with $\BranchEvery{=}16$ makes its first prune near mask ratio $0.9$, where its ROC-AUC is lowest (Figure~\ref{fig:mechanism}(a)), putting the chain that eventually becomes $t{=}3$ at risk; the offline counterfactual in App.~\ref{app:additional_matched_controls} finds that a top-$1$ cut at the initial stored state removes every correct lineage in $46\%$ of cases.
(iii)~\emph{Causal PRM miscalibration} (\Cref{sec:mechanism-causal}): at mask${=}0$ the causal PRM assigns its highest score ($+1.001$) to $t{=}13$, a short trajectory that reaches a wrong answer. Top-$1$ selection within this triple picks the wrong answer, a concrete case of pooled discrimination failing to carry over to within-problem ranking (Proposition~\ref{prop:pooled_auc_rerank}).

\paragraph{What ORM Rerank does differently on the same candidates.}
The bidirectional ORM scores $t{=}3$ highest among the $N{=}32$ Vanilla candidates for this problem and selects it. The causal failure is common: on $639$ of the $1{,}319$ test problems ($48\%$), the causal PRM confidently prefers a wrong trajectory while a correct one is available. Across all problems, ORM Rerank$@8$ leads the mean-pooled causal reranker by $31.2$\,pp, $75.13\%$ against $43.90\%$ (App.~\ref{app:lasttoken_pareto}).

\section{Broader Impact, Ethics, and Compute Resources}
\label{app:broader_impact_compute}

This paper is diagnostic: it evaluates reward-guided dLLM reasoning protocols and recommends matched-compute baselines and readout checks. The positive impact is more reproducible reward-model evaluation and less wasted test-time compute; negative risk is limited to ordinary math and code reasoning evaluation of LLMs, with no new deployment capability or safety-sensitive dataset.

\paragraph{Ethics.}
All experiments use public benchmarks (GSM8K, MATH, MBPP) and public model backbones (Dream-7B, LLaDA-8B-Base). No human subjects, personal data, private data, or crowdsourcing are involved. The released artifacts are evaluation resources and LoRA adapters rather than base-model weights.

\paragraph{Compute resources.}
The main experimental campaign used at most $11\times$ NVIDIA H20-96GB GPUs over roughly one month. A PRM training job takes about $7$ GPU-hours; recorded job durations put a $32$-sample MATH500 best-of-$N$ sweep at about $48$ GPU-hours per run and the GSM8K $32$-sample sweep at about $100$ GPU-hours under our sharding and batching. PRM and ORM training, snapshot scoring, and the GSM8K and MATH500 sweeps over runs, $K$ values, and $\BranchEvery$ values, including preliminary experiments, used about $2{,}400$ H20-96GB GPU-hours: about $960$ for the final $K\in\{6,12,24\}$ campaign and about $1{,}400$ for earlier training and evaluation. Per-method wall-clock validation for the headline comparisons is reported in App.~\ref{app:compute}.

\paragraph{Reproducibility.}
The released artifacts include the scored snapshot corpus, trained PRM adapters/checkpoints, evaluation toolkit with Pareto analysis, figure scripts, and training/evaluation recipes. The public code repository is \url{https://github.com/dLLM-PRM-Gap/}; weights and data are collected at \url{https://huggingface.co/collections/YanZhanPKU/dllm-prm-gap}.


\end{document}